\documentclass{article}

\newif\ifdraft
\draftfalse

\usepackage{ijcai13}
\usepackage{times}
\usepackage{latexsym,amssymb,amsmath}
\usepackage{xspace}
\usepackage{url}
\usepackage{microtype}
\usepackage{comment}
\usepackage{amsthm}
\usepackage{enumerate}
\usepackage[defblank]{paralist}
\usepackage{hyperref}

\newcommand{\cqa}[2][]{\mathit{cqa}_{#1}(#2)}
\newcommand{\Cqa}[2][]{\textsc{cqa}_{#1}(#2)}

\newcommand{\temp}{\mathsf{State}(\mathit{temp})}

\newcommand{\funcCall}{\F} 

\newcommand{\serviceCall}{\mathbb{SC}} 
\newcommand{\serviceCallMap}{m} 
\newcommand{\servCall}{\textsc{serviceCalls}} 

\newcommand{\Rep}{\ensuremath{\textsc{rep}}\xspace}

\newtheorem{counter}{Counter}[section] 
\newtheorem{lemma}[counter]{Lemma}

\newtheorem{example}[counter]{Example}
\newtheorem{theorem}[counter]{Theorem}

\usepackage[normalem]{ulem} 
\ifdraft
  \usepackage{mychanges}
  \usepackage[colorinlistoftodos]{todonotes}
\else
  \usepackage[final]{mychanges}
  \usepackage[disable]{todonotes}
\fi
\setauthormarkup[right]{{\scriptsize[#1]}}
\definechangesauthor{DC}{orange}
\definechangesauthor{EK}{brown}
\definechangesauthor{MM}{green}
\definechangesauthor{AS}{blue}
\definechangesauthor{DZ}{red}

\newcommand{\asd}[1]{\deleted[AS]{#1}} 
\newcommand{\asr}[2]{\replaced[AS]{#1}{#2}} 

\newcommand{\tbox}{\T} 
\newcommand{\initABox}{A_0} 
\newcommand{\actSet}{\Gamma} 
\newcommand{\action}{\gamma} 
\newcommand{\procSet}{\Pi} 
\newcommand{\process}{\pi} 

\def\qed{\hfill{\qedboxempty}} 
\def\qedboxempty{\vbox{\hrule\hbox{\vrule\kern3pt
\vbox{\kern3pt\kern3pt}\kern3pt\vrule}\hrule}}

\newcommand{\A}{\mathcal{A}} \newcommand{\B}{\mathcal{B}}
\newcommand{\C}{\mathcal{C}} 
 \newcommand{\F}{\mathcal{F}}
 \renewcommand{\H}{\mathcal{H}}
\newcommand{\I}{\mathcal{I}} 
\newcommand{\K}{\mathcal{K}} \renewcommand{\L}{\mathcal{L}}

\renewcommand{\S}{\mathcal{S}} \newcommand{\T}{\mathcal{T}}

\newcommand{\mf}{\mathfrak}

\newcommand{\eset}{\emptyset}

\newcommand{\ra}{\rightarrow}

\newcommand{\per}{\mbox{\bf .}}                  

\newcommand{\set}[1]{\{#1\}}                      

\newcommand{\tup}[1]{\langle #1\rangle}            

\newcommand{\myi}{\emph{(i)}\xspace}
\newcommand{\myii}{\emph{(ii)}\xspace}

\newcommand{\EXPTIME}{\textsc{ExpTime}\xspace}

\newcommand{\dom}[1][\I]{\Delta^{#1}}  
\newcommand{\Int}[2][\I]{#2^{#1}}      

\newcommand{\SOMET}[1]{\exists #1}

\newcommand{\INV}[1]{#1^{-}}

\newcommand{\BOX}[1]{ [\!-\!] #1}
\newcommand{\DIAM}[1]{\langle \!-\! \rangle #1}

\newcommand{\ISA}{\sqsubseteq}

\newcommand{\limp}{\rightarrow}

\newcommand{\MOD}[1]{(#1)^{\mf{}}}

\newcommand{\MODA}[1]{(#1)_{\vfo,\vso}^{\mf{}}}   

\newcommand{\MODAX}[2]{(#1)_{\vfo #2,\vso}^{\mf{}}}   

\newcommand{\true}{\mathsf{true}}
\newcommand{\false}{\mathsf{false}}

\newcommand{\dllitea}{\textit{DL-Lite}\ensuremath{_{\mathcal{A}}}\xspace}

\newcommand{\ans}[2][]{\mathit{ans}_{#1}(#2)}
\newcommand{\rew}[1]{\mathit{rew}(#1)}

\newcommand{\conj}{\mathit{conj}}

\newcommand{\map}[2]{#1 \rightsquigarrow #2}
\newcommand{\carule}[2]{#1 \mapsto #2}

\renewcommand{\A}{A}
\renewcommand{\T}{T}

\newcommand{\ex}[1]{\mathsf{#1}}

\newcommand{\const}[1]{\C_{#1}}

\newcommand{\funct}[1]{(\ex{funct}~#1)}

\newcommand{\adom}[1]{\textsc{adom}(#1)}

\newcommand{\DO}[1]{\textsc{do}(#1)}

\newcommand{\Ans}[2][]{\textsc{ans}_{#1}(#2)}
\renewcommand{\=}{\hspace{-0.08cm}=\hspace{-0.08cm}}

\newcommand{\abox}{\mathit{abox}}

\newcommand{\CONST}{\C}

\newcommand{\FUNC}{\ensuremath{\F}\xspace}   

\newcommand{\HERBRAND}{\ensuremath{\mathbb{U}}\xspace}

\newcommand{\idb}{\ensuremath{\I_0}\xspace}

\newcommand{\pos}[1]{\ensuremath{#1^+}\xspace}

\newcommand{\act}{\alpha}
\newcommand{\sys}{\ensuremath{\S}\xspace}
\newcommand{\dcds}{\mbox{DCDS}\xspace}

\newcommand{\muladom}{\ensuremath{\muL_{A}^{{\textnormal{EQL}}}}\xspace}
\newcommand{\muladomcqa}{\ensuremath{\muL_{A}^{{\textnormal{CQA}}}}\xspace}
\newcommand{\muladomit}{\ensuremath{\muL_{A}^{\textnormal{IT}}}\xspace}

\newcommand{\hbsim}{\approx}

\newcommand{\restrict}[2]{\ensuremath{#1 \vert_{#2}}}

\newcommand{\muL}{\mu\L} 

\newcommand{\rmap}{\ensuremath{m}\xspace}

\newcommand{\exec}{\textsc{exec}}
\newcommand{\pexec}{\textsc{p-exec}}

\newcommand{\doo}[3]{\textsc{do}^{#1}_{#2}(#3)}

\newcommand{\groundexec}[2]{{\textsc{evals}_{#1}(#2)}}

\newcommand{\skolems}[1]{{\textsc{calls}({#1})}}

\newcommand{\trans}{\Rightarrow}

\newcommand{\iconst}{\CONST_0}

\newcommand{\cts}[1]{\Upsilon_{#1}^\textnormal{S}}
\newcommand{\pcts}[1]{\Theta_{#1}^\textnormal{S}}

\newcommand{\bs}[1]{\Upsilon_{#1}^{b}\xspace}
\newcommand{\cs}[1]{\Upsilon_{#1}^{c}\xspace}

\newcommand{\verit}[2]{#1 \models_\textnormal{R} #2}
\newcommand{\vercqa}[2]{#1 \models_\textnormal{CQA} #2}

\newcommand{\skab}{S-KAB}

\newcommand{\ckab}{CQA-KAB}

\newcommand{\cittransl}{\tau}

\newcommand{\qunsatf}{q^f_{\textnormal{unsat}}}
\newcommand{\qunsatn}{q^n_{\textnormal{unsat}}}
\newcommand{\lab}{\textsc{label}}
\newcommand{\viol}{\textsc{viol}}

\newcommand{\sts}[1]{\Theta_#1}

\renewcommand{\mf}[1]{\Upsilon_{#1}}

\newcommand{\scset}{\ensuremath{\mathbb{SC}}\xspace}

\newcommand{\domain}[1]{\ensuremath{\textsc{dom}(#1)}\xspace}
\newcommand{\image}[1]{\ensuremath{\textsc{im}(#1)}\xspace}

\newcommand{\cstate}{\mbox{c-state}\xspace}

\newcommand{\vfo}{\ensuremath{v}}
\newcommand{\vso}{\ensuremath{V}}

\begin{document}

\title{
  Verification of Inconsistency-Aware Knowledge and Action Bases\\
(Extended Version)%
 \thanks{The authors are supported by the EU project ACSI (FP7-ICT-257593) and
  Optique (FP7-IP-318338).  Kharlamov was also supported by the ERC grant
  Webdam, agreement n.~226513.}}

\author{Diego Calvanese, Evgeny Kharlamov, Marco Montali, Ario Santoso, Dmitriy
 Zheleznyakov\\
 KRDB Research Centre for Knowledge and Data\\
 Free University of Bozen-Bolzano\\
 \textit{lastname}@inf.unibz.it
}

\maketitle


\begin{abstract} 



Description Logic Knowledge and Action Bases (KABs)
have been recently introduced as a mechanism that provides a
semantically rich representation of the information on the domain of
interest in terms of a DL KB and a set of actions to change such
information over time, possibly introducing new objects. In this
setting, decidability of verification of sophisticated temporal properties over KABs,
expressed in a variant of first-order $\mu$-calculus, has been shown.
However, the established framework treats inconsistency in a
simplistic way, by rejecting inconsistent states produced through
action execution. We address this problem by showing how inconsistency handling based on the notion of repairs can be
integrated into KABs, resorting to inconsistency-tolerant semantics.
In this setting, we establish decidability and complexity of verification.
\end{abstract}


\section{Introduction}
\label{sec:introduction}

Recent work in knowledge representation and databases has addressed
the problem of dealing with the combination of knowledge, processes
and data in the design of complex enterprise systems \cite{DHPV09,Vian09,BCDD*12,CDLMS12,LDLHV12}.
The verification of temporal properties in this setting represents a
significant research challenge, since data and knowledge makes the system infinite-state, and
neither finite-state model checking \cite{ClGP99} nor most of the current
techniques for infinite-state model checking \cite{BCMS01} apply to this case.

Along this line,
\emph{Knowledge and Action Bases} (\emph{KABs}) \cite{BCDD*12} have
have been recently introduced as a mechanism that provides a
semantically rich representation of the information on the domain of
interest in terms of a Description Logic (DL) KB and a set of actions to change such
information over time, possibly introducing new objects. In this
setting, decidability of verification of sophisticated temporal properties over KABs,
expressed in a variant of first-order $\mu$-calculus, has been shown.

However, KABs and the majority of approaches dealing with verification
in this complex setting assume a rather simple treatment of inconsistency resulting as an
effect of action execution: inconsistent states are simply
rejected (see, e.g., \cite{DeSV07,DHPV09,BCDDM13}). In general, this is not satisfactory, since the
inconsistency may affect just a small portion of the entire KB, and
should be treated in a more careful way.
Starting from this observation, in this work we leverage on the research on instance-level evolution of knowledge bases
\cite{Wins90,EiGo92,FMKPA08,CKNZ10b}, and, in particular, on the
notion of knowledge base repair \cite{LLRRS10}, in order to make KABs
inconsistency-aware. In particular, we present a novel setting that extends KABs by assuming the
availability of a repair service that is able to compute, from an
inconsistent knowledge base resulting from the execution of an action,
one or more \emph{repairs}, in which the inconsistency has been
removed with a ``minimal'' modification to the existing
knowledge. This allows us to incorporate, in the temporal
verification formalism, the possibility of quantifying over repairs.
Notably, our novel setting is able to deal with both deterministic
semantics for repair, in which a single repair is computed from an
inconsistent knowledge base, and non-deterministic ones, by
simultaneously taking into account all possible repairs. 
We show how the techniques developed for KABs extend to this
inconsistency-aware setting, preserving both decidability and
complexity results, under the same assumptions required in KABs for decidability.

We also show how our setting is able to accommodate meta-level information about
the sources of inconsistency at the intentional level, so as to allow
them to be queried when verifying temporal properties of the
system. The decidability and complexity results for verification carry
over to this extended setting as well.





The proofs of all presented theorems are contained in the appendix.


\section{Preliminaries}
\label{sec:prelim}

\subsection{\dllitea Knowledge Bases}
For expressing knowledge bases, we use \dllitea
\cite{PLCD*08,CDLL*09}.  The syntax of \emph{concept} and \emph{role}
\emph{expressions} in \dllitea is as follows 
{\small{
\[
   \begin{array}{rcl@{}l}
     B &\longrightarrow&  N &~\mid~ \SOMET{R} \\
   \end{array}
   \qquad\qquad
   \begin{array}{rcl@{}l}
     R &\longrightarrow&  P &~\mid~ \INV{P}\\
   \end{array}
\]
}}
where $N$ denotes a \emph{concept name},
$P$ a \emph{role name}, and $\INV{P}$ an \emph{inverse role}.
%
A \dllitea \emph{knowledge base} (KB) is a pair $(T,A)$, where:
\begin{inparaenum}[\it (i)]
\item $A$ is an Abox, i.e., a finite set of \emph{ABox membership assertions} of the form $ N(t_1) \mid P(t_1,t_2)$, where
  $t_1$, $t_2$ denote individuals
\item $T$ is a TBox, i.e., $T = T_p \uplus
 T_n \uplus T_f$, with $T_p$ a finite set of \emph{positive inclusion
   assertions} of the form $B_1 \ISA B_2$, $T_n$ a finite set of \emph{negative inclusion
   assertions} of the form $B_1 \ISA \neg B_2$, and $T_f$ a finite set
 of \emph{functionality assertions} of the form $\funct{R}$.


\end{inparaenum}


We adopt the standard FOL semantics of DLs based on FOL interpretations
$\I=(\dom,\Int{\cdot})$ such that $\Int{c}\in\dom$, $\Int{N}\subseteq\dom$, and
$\Int{P}\subseteq\dom\times\dom$.  The semantics of the construct, of TBox and
ABox assertions, and the notions of \emph{satisfaction} and of \emph{model} are
as usual.
We also say that $A$ is \emph{$T$-consistent} if $(T,A)$ is
satisfiable, i.e., admits at least one model, otherwise we say $A$ is
\emph{$T$-inconsistent}.

\smallskip
\noindent
\textbf{Queries.}
As usual (cf.\ \asr{OWL~2~QL}{OWL~2}), answers to queries are formed by terms denoting
individuals explicitly mentioned in the ABox.  The \emph{domain of an ABox}
$A$, denoted by $\adom{A}$, is the (finite) set of terms appearing in $A$.
A \emph{union of conjunctive queries} (UCQ) $q$ over a KB $(T,A)$ is a FOL
formula of the form
$\bigvee_{1\leq i\leq n}\exists\vec{y_i}\per\conj_i(\vec{x},\vec{y_i})$ with free
variables $\vec{x}$ and existentially quantified variables
$\vec{y}_1,\ldots,\vec{y}_n$.  Each $\conj_i(\vec{x},\vec{y_i})$ in
$q$ is a conjunction of atoms of the form $N(z)$, $P(z,z')$, where $N$
and $P$ respectively denote a concept and a role name occurring in
$T$, and $z$, $z'$ are constants in $\adom{A}$ or variables in
$\vec{x}$ or $\vec{y_i}$, for some $i\in\{1,\ldots,n\}$.
%
The \emph{(certain) answers} to $q$ over $(T,A)$ is the set $\ans{q,T,A}$ of
substitutions 
$\sigma$ of the free variables of $q$ with constants in $\adom{A}$ such that
$q\sigma$ evaluates to true in every model of $(T,A)$.
If $q$ has no free variables, then it is called \emph{boolean} and its certain
answers are either $\true$ or $\false$.

We compose UCQs using ECQs, i.e., queries of
the query language \textit{EQL-Lite}(UCQ)~\cite{CDLLR07b}, which is
the FOL query language whose atoms are UCQs evaluated according to the
certain answer semantics above. An \emph{ECQ} over $T$ and
$A$ is a possibly open formula of the form
\[
  Q ~:=~
  [q] ~\mid~ 
\lnot Q ~\mid~ Q_1\land Q_2 ~\mid~
  \exists x\per Q
\]
where $q$ is a UCQ.
%
The \emph{answer to $Q$ over $(\T,\A)$}, is the set $\Ans{Q,\T,\A}$ of tuples
of constants in $\adom{A}$ defined by composing  the certain answers
$\ans{q,T,A}$ of UCQs $q$ through first-order
constructs, and
interpreting existential variables as
ranging over $\adom{A}$.

Finally, we recall that \dllitea enjoys the \emph{FO rewritability}
property, which states that for every UCQ
$q$, $\ans{q,T,A} = \ans{rew(q),\emptyset,A}$, where $rew(q)$ is a UCQ
computed by the reformulation algorithm in \cite{CDLL*09}. Notice
that this algorithm can be extended to ECQs \cite{CDLLR07b}, and that
its effect is to ``compile away'' the TBox. 

\subsection{Knowledge and Action Bases}
\label{sec-framework}
We recall the notion of \emph{Knowledge and Action Bases (KABs)}, as
introduced in \cite{BCDD*12}. In the following, we make use of a
countably infinite set $\const{}$ of constant to denote all possible
value in the system.  Moreover, we also make use of a finite set
$\funcCall$ of functions that represent service calls, and can be used
to inject fresh values into the system.

A KAB is a tuple
\asr{$\K=(\tbox,\initABox,\actSet,\procSet)$}{$\K=(\T,A_0,\Gamma,\Pi)$}
where \asr{$\tbox$}{$T$} and \asr{$\initABox$}{$A_0$} form the
knowledge base (KB), and \asr{$\actSet$}{$\Gamma$} and
\asr{$\procSet$}{$\Pi$} form the \asd{or} action base. Intuitively,
the KB maintains the information of interest. It is formed by a fixed
\dllitea TBox $\T$ and an initial \emph{$T$-consistent} \dllitea ABox
\asr{$\initABox$}{$A_0$}. \asr{$\initABox$}{$A_0$} represents the
initial state of the system and, differently from \asr{$\tbox$}{$\T$},
it evolves and incorporates new information from the external world by
executing actions \asr{$\actSet$}{$\Gamma$}, according to the
sequencing established by process \asr{$\procSet$}{$\Pi$}.
\asr{$\actSet$}{$\Gamma$} is a finite set actions.  An \emph{action}
\asr{$\action \in \actSet$}{$\gamma\in\Gamma$} modifies the current
ABox $A$ by adding or deleting assertions, thus generating a new ABox
$A'$.  \asr{$\action$}{$\gamma$} is constituted by a signature and an
effect specification.  The \emph{action signature} is constituted by a
name and a list of individual \emph{input parameters}.  Such
parameters need to be instantiated with individuals for the execution
of the action.  Given a substitution $\theta$ for the input
parameters, we denote by \asr{$\action\theta$}{$\gamma\theta$} the
instantiated action with the \emph{actual} parameters coming from
$\theta$.  The \emph{effect specification} consists of a set
$\{e_1,\ldots,e_n\}$ of effects, which take place simultaneously.  An
\emph{effect} $e_i$ has the form $\map{[q^+_i] \land Q^-_i}{A'_i}$,
where:
  \begin{inparaenum}[\it (i)]
  \item $q^+_i$ is an UCQ, and $Q^-_i$ is an arbitrary ECQ whose free
    variables occur all among the free variables of
    $q^+_i$;
  \item $A'_i$ is a set of facts (over the alphabet of $T$) which
    include as terms: individuals in $A_0$, free variables of $q^+_i$,
    and Skolem terms $f(\vec{x})$ having as arguments free variables
    $\vec{x}$ of $q^+_i$.
  \end{inparaenum}
The distinction between $q_i^+$ and $Q_i^-$ is needed
for technical reasons (see Appendix~\ref{sec:wa-kabs}). 

  The process \asr{$\procSet$}{$\Pi$} is a finite set of
  condition/action rules.  A \emph{condition/action rule}
  \asr{$\process\in\procSet$}{$\pi\in\Pi$} is an expression of the
  form \asr{$\carule{Q}{\action}$}{$ \carule{Q}{\alpha} $}, where \asr{$\action$}{$\alpha$} is an action in \asr{$\actSet$}{$\Gamma$}
  and $Q$ is an ECQ over $T$, whose free variables are exactly the
  parameters of \asr{$\action$}{$\gamma$}.  The rule expresses that, for each tuple
  $\sigma$ for which condition $Q$ holds, the action \asr{$\action$}{$\alpha$} with
  actual parameters $\sigma$ \emph{can} be executed.

{\footnotesize
\begin{example}\label{RunningExample}
\asr{$\K = (\tbox, \initABox, \actSet, \procSet)$}{$\K=(\T,A_0,\Gamma,\Pi)$} is a KAB defined as follows:
\begin{inparaenum}[\it (i)]
\item $\tbox = \{ C \sqsubseteq \neg D\}$,
\item \asr{$\initABox = \{C(a)\}$}{$A_0 = \{C(a)\}$}, 
\item \asr{$\actSet = \{\action_1, \action_2\}$}{$\Gamma = \{\gamma_1, \gamma_2\}$} with
\asr{$\action_1():\{ \map{C(x)}{D(x), C(x)} \}$}{$\gamma_1():\{ \map{C(x)}{D(x), C(x)} \}$} and
\asr{$\action_2(p):\{ \map{C(p)}{G(f(p))} \}$}{$\gamma_2(p):\{ \map{C(p)}{G(f(p))} \}$}, 
\item \asr{$\Pi =\{true \mapsto \action_1, C(y) \mapsto \action_2(y)\}$}{$\Pi =\{true \mapsto \gamma_1, C(y) \mapsto \gamma_2(y)\}$}.
\end{inparaenum}
\qed
\end{example}
}

\section{Verification of Standard KABs}
\label{sec:verification1}
We are interested in verifying temporal/dynamic properties over
KABs. To this aim, we fix a countably infinite set $\CONST$ of
individual constants (also called values), which act as standard names,
and finite set of distinguished constants
$\iconst \subset \CONST$. Then, we define the execution semantics of a KAB in
terms of a possibly infinite-state \emph{transition system}. More
specifically, we consider transition systems of the form
$(\CONST,\T,\Sigma,s_0,\abox,{\Rightarrow})$, where:
\begin{inparaenum}[\it (i)]
\item \asr{$\tbox$}{$\T$} is a TBox;
\item $\Sigma$ is a set of states;
\item $s_0 \in \Sigma$ is the initial state;
\item $\abox$ is a function that, given a state $s\in\Sigma$, returns an ABox
  associated to $s$, which has as individuals values of $\CONST$ and
  conforms to $\T$;
\item ${\Rightarrow} \subseteq \Sigma\times\Sigma$ is a transition
  relation between pairs of states.
\end{inparaenum}

The standard execution semantics for a KAB \asr{$\K = (\tbox,
  \initABox, \actSet, \procSet)$}{$\K=(\T,A_0,\Gamma,\Pi)$} is
obtained starting from \asr{$\initABox$}{$A_0$} by nondeterministically applying every executable
action with corresponding legal parameters, and considering each possible value returned by applying the
involved service calls. Notice that this is radically different from
\cite{BCDD*12}, where service calls are not evaluated when
constructing the transition system.
The executability of an action with fixed parameters does not only
depend on the process \asr{$\procSet$}{$\Pi$}, but also on the $\T$-consistency of the
ABox produced by the application of the action: if the resulting ABox
is $\T$-inconsistent, the action is considered as non
executable with the chosen parameters.

We consider
\emph{deterministic} services, i.e., services that return always the
same value when called with the same input
parameters. Nondeterministic services can be seamlessly added without
affecting our technical results.
To ensure that services behave deterministically, we recast the
approach in \cite{BCDDM13} to the semantic setting of KABs, keeping track, in the states of the
transition system generated by $\K$, of all the service call
results accumulated so far. To do so, we introduce the set of (Skolem terms representing)
service calls as $\scset = \{f(v_1,\ldots,v_n) \mid f/n \in \FUNC
\textrm{ and } \{v_1,\ldots,v_n\} \subseteq \CONST \}$, and define a
\emph{service call map} as a partial function
$\rmap:\scset\ra\CONST$.

A  \emph{state} of the transition system generated by $\K$
is a pair $\tup{\A, \rmap}$,
where $\A$  is an ABox and \rmap is a  service call map.
Let \asr{$\action(p_1,\ldots,p_r) :
  \set{e_1,\ldots,e_k}$}{$\alpha(p_1,\ldots,p_r) :
  \set{e_1,\ldots,e_k}$} be an action in \asr{$\actSet$}{$\Gamma$}
with parameters $p_1,\ldots,p_r$, and $e_i=\map{[q_i^+]\land
  Q_i^-}{E_i}$.  Let $\sigma$ be a substitution for $p_1,\ldots,p_r$
with values taken from $\CONST$.  We say that $\sigma$ is \emph{legal}
for \asr{$\action$}{$\alpha$} in state $\tup{\A,\rmap}$ if there
exists a condition-action rule \asr{$Q \mapsto \action$}{$Q \mapsto
  \alpha$} in \asr{$\procSet$}{$\Pi$} such that
$\tup{p_1,\ldots,p_r}\sigma \in \Ans{Q,\A}$.
We denote with \asr{$\doo{}{}{\tbox,\A,
    \action\sigma}$}{$\doo{}{}{\T,\A, \alpha\sigma}$} the set of facts
obtained by evaluating the effects of action \asr{$\action$}{$\alpha$} with parameters
$\sigma$ on ABox $\A$, so as to \emph{progress} (cf.~planning \cite{GhNT04})
the system from the
current state to the next:
\[
  \doo{}{}{\T, \A, \action\sigma} = \bigcup_{\map{[q_i^+]\land
    Q_i^-}{E_i} \text{ in } \action\ }
  \bigcup_{\rho\in\Ans[]{([q_i^+]\land Q_i^-)\sigma,\tbox,\A}} E_i\sigma\rho
\]
%
%
%
The returned set is the union of the results of
    applying the effects specifications in \asr{$\action$}{$\alpha$}, where the
    result of each effect specification $\map{[q_i^+]\land Q_i^-}{E_i}$ is,
    in turn, the set of facts $E_i\sigma\rho$ obtained from
    $E_i\sigma$ grounded on all the assignments $\rho$ that satisfy
    the query $[q_i^+]\land Q_i^-$ over $A$.

    Note that $\doo{}{}{}$ generates facts that use values from the
    domain $\CONST$, but also Skolem terms, which model service calls.
    For any such set of facts $E$, we denote with $\skolems{E}$ the
    set of calls it contains, and with
    \asr{$\groundexec{}{\tbox,\A,\action\sigma}$}{$\groundexec{}{\T,\A,\alpha\sigma}$}
    the set of substitutions that replace all service calls in
    \asr{$\doo{}{}{\tbox,\A, \action\sigma}$}{$\doo{}{}{\T,\A, \alpha\sigma}$} with values in $\CONST$:
\begin{tabbing}
$\groundexec{}{\tbox,\A,\action\sigma}= \{ \theta\ |\ $ \=
                        $\theta \mbox { is a total function }$\+\\
                        $\theta: \skolems{\doo{}{}{\tbox, \A, \action\sigma}} \ra \CONST \}$.
\end{tabbing}
Each substitution in \asr{$\groundexec{}{\T,\A,\action\sigma}$}{$\groundexec{}{\T,\A,\alpha\sigma}$}
models the simultaneous evaluation of all service calls, returning
results arbitrarily chosen from $\CONST$.

{\footnotesize
\begin{example}
  \label{ex:exec}
  Consider our running example (Example~\ref{RunningExample}). Starting
  from \asr{$\initABox$}{$A_0$}, the execution of \asr{$\action_1$}{$\gamma_1$} would produce $A' = \{D(a), C(a)\}$,
  which is $T$-inconsistent.  Thus, the execution of
  \asr{$\action_1$}{$\gamma_1$} in \asr{$\initABox$}{$A_0$} should
  either be rejected or its effect should be repaired (cf.~Section~\ref{sec:KABWithInconsistencyTolerance}). The execution of \asr{$\action_2$}{$\gamma_2$} with legal parameter $a$ instead produces $A''
  = \{G(c)\}$ when the service call $f(a)$ returns $c$. $A''$ is
  $T$-consistent, and \asr{$\action_2(a)$}{$\gamma_2(a)$} is therefore executable in \asr{$\initABox$}{$A_0$}. \qed
\end{example}
}
Given a KAB \asr{$\K =(\tbox,\initABox,\actSet,\procSet)$}{$\K =(\T,A_0,\Gamma,\Pi)$}, we employ $\doo{}{}{}$ and $\groundexec{}{}$ to define a transition relation $\exec_\K$
connecting two states through the application of an action with
parameter assignment. In particular, given an action with parameter assignment
\asr{$\action\sigma$}{$\alpha\sigma$}, we have
\asr{$\tup{\tup{\A,\rmap},\action\sigma,\tup{\A',\rmap'}} \in \exec_\K$}{$\tup{\tup{\A,\rmap},\alpha\sigma,\tup{\A',\rmap'}} \in \exec_\K$} if
the following holds:
\begin{inparaenum}[\it (i)]
\item $\sigma$ is a legal parameter assignment for \asr{$\action$}{$\alpha$} in state
  $\tup{\A,\rmap}$, according to \asr{$\procSet$}{$\Pi$};
\item there exists \asr{$\theta \in \groundexec{}{\tbox,\A,\action\sigma}$}{$\theta \in \groundexec{}{\T,\A,\alpha\sigma}$}
  such that $\theta$ and $\rmap$ agree on the
  common values in their domains (so as to realize the deterministic
  service semantics);
\item \asr{$\A' = \doo{}{}{\T, \A, \action\sigma}\theta$}{$\A' = \doo{}{}{\T, \A, \alpha\sigma}\theta$};
\item $\rmap' = \rmap \cup \theta$ (i.e., the history of issued service calls
  is updated).
\end{inparaenum}

\smallskip
\noindent {\bf Standard transition system.}
 The \emph{standard transition system} $\cts{\K}$ for KAB \asr{$\K =(\tbox,\initABox,\actSet,\procSet)$}{$\K =(\T,\A_0,\Gamma,\Pi)$}  is a
  (possibly infinite-state) transition system
  $(\CONST,\T,\Sigma,s_0,\abox,\trans)$ where:
\begin{inparaenum}[\it (i)]
\item \asr{$s_0 = \tup{\initABox,\emptyset}$}{$s_0 = \tup{\A_0,\emptyset}$}; \item $\abox(\tup{\A,\rmap}) = \A$;
\item  $\Sigma$ and
$\trans$ are defined by simultaneous induction
as the smallest sets satisfying the following properties:
\myi $s_0 \in \Sigma$; \myii if $\tup{\A,\rmap} \in \Sigma$ , then for
all actions \asr{$\action$}{$\alpha$} in \asr{$\actSet$}{$\Gamma$}, for all
substitutions $\sigma$ for the parameters
  of \asr{$\action$}{$\alpha$}
and for all
  $\tup{\A',\rmap'}$ such that $\A'$ is \emph{$\T$-consistent} and
  \asr{$\tup{\tup{\A,\rmap},\action\sigma,\tup{\A',\rmap'}}\in \exec_\K$}{$\tup{\tup{\A,\rmap},\alpha\sigma,\tup{\A',\rmap'}}\in \exec_\K$},
 we have $\tup{\A',\rmap'}\in\Sigma$ and $\tup{\A,\rmap} \trans \tup{\A',\rmap'}$.
We call \skab\ a KAB interpreted under the standard execution semantics.
\end{inparaenum}

{\footnotesize
\begin{example}
\label{ex:ts}
Consider $\K$ of Example~\ref{RunningExample} and its standard transition
system $\cts{\K}$. As discussed in Example~\ref{ex:exec}, in state
\asr{$s_0 = \tup{\initABox,\emptyset}$}{$s_0 = \tup{A_0,\emptyset}$}
only \asr{$\action_2$}{$\gamma_2$} is applicable with parameter
$a$. Since \asr{$\doo{}{}{\T,A_0,\action_2(a)} = \{G(f(a))\}$}{$\doo{}{}{\T,A_0,\gamma_2(a)} = \{G(f(a))\}$},
  $\cts{\K}$ contains infinitely many successors for $s_0$, each of
the form $\tup{\{G(x)\},\{f(a) \mapsto x\}}$, where $x$ is arbitrarily
substituted with a specific value picked from $\CONST$.  \qed
\end{example}
}

\smallskip
\noindent
\textbf{Verification Formalism.}~
To specify sophisticated temporal properties over KABs, we resort to
the first-order variant of $\mu$-calculus
\cite{Stir01,Park76} defined in \cite{BCDD*12}.
This variant, here called $\muladom$, exploits EQL to query the states, and supports a particular form of first-order
quantification across states: quantification ranges over the
individuals explicitly present in the current active domain, and can
be arbitrarily referred to in later states of the systems.
Formally, \muladom is defined as follows:
\[
  \Phi ~:=~ Q ~\mid~ \lnot \Phi ~\mid~ \Phi_1 \land \Phi_2
  ~\mid~ \exists x.\Phi ~\mid~ \DIAM{\Phi} ~\mid~ Z ~\mid~ \mu Z.\Phi
\]
where $Q$ is a possibly open EQL query that can make use of the
distinguished constants in $\iconst$, and $Z$ is a second order predicate
variable (of arity 0).
We make use of the following abbreviations: $\forall x.  \Phi = \neg
(\exists x.\neg \Phi)$, $\Phi_1 \lor \Phi_2 = \neg (\neg\Phi_1 \land
\neg \Phi_2)$, $\BOX{\Phi} = \neg \DIAM{\neg \Phi}$, and $\nu Z. \Phi
= \lnot\mu Z. \neg \Phi[Z/\neg Z]$.

The semantics of \muladom formulae is defined over transition
systems $\tup{\const{}, \T,\Sigma,s_0,\abox,{\Rightarrow}}$.  Since \muladom
contains formulae with both individual and predicate free variables, given
a transition system $\mf{}$, we introduce an individual variable valuation
$\vfo$, i.e., a mapping from individual variables $x$ to $\const{}$, and a
predicate variable valuation $\vso$, i.e., a mapping from the predicate
variables $Z$ to a subset of $\Sigma$.
All the language primitives follow
the standard $\mu$-calculus semantics, apart from the two listed
below \cite{BCDD*12}:
\[
  \begin{array}{r@{\ }l@{\ }l@{\ }l}
    \MODA{Q} & = &\{s \in \Sigma\mid& \Ans{Q\vfo, T, \abox(s)} = \mathit{true}\}\\
   \MODA{\exists x. \Phi} & =&\{s \in \Sigma \mid&\exists d. d \in \adom{\abox(s)}
    \\&&&\mbox{ and } s \in \MODAX{\Phi}{[x/d]}\}\\
 \end{array}
\]  
Here,  $Q\vfo$ stands for the query obtained from $Q$ by substituting its free
variables according to $\vfo$.
%
%
When $\Phi$ is a closed formula, $\MODA{\Phi}$ does not depend on
$\vfo$ or $\vso$, and we denote the extension of $\Phi$ simply by
$\MOD{\Phi}$.
A closed formula $\Phi$ holds in a state $s
\in \Sigma$ if $s \in \MOD{\Phi}$.
We call
\emph{model checking} verifying whether $s_0 \in \MOD{\Phi}$, and we
write in this case $\mf{}\models \Phi$.


\smallskip
\noindent
\textbf{Decidability of verification.}~
We are interested in studying the verification of $\muladom$
properties over \skab s.
We can easily recast the undecidability result in \cite{BCDD*12}
to the case of \skab s, obtaining that verification is undecidable even for the very simple temporal reachability property $\mu Z.(\textsf{N}(a) \lor
\DIAM{Z})$, with $\textsf{N}$ atomic concept and $a \in \CONST$.

Despite this undecidability result, we can isolate an interesting class of KABs that enjoys
verifiability of arbitrary $\muladom$ properties through finite-state abstraction. This class is based on a semantic restriction
named \emph{run-boundedness} \cite{BCDDM13}. Given an \skab\ $\K$, a run $\tau =
s_0s_1\cdots$ of $\cts{\K}$ is \emph{bounded} if
there exists a finite bound \emph{b} s.t.\ $\left|\bigcup_{s \text{
      state of } \tau } \adom{\abox(s)} \right| < b$. We say that $\K$ is \emph{run-bounded} if there
exists a bound \emph{b} s.t. every run $\tau$ in $\cts{\K}$ is bounded by
\emph{b}.
\begin{theorem}
\label{thm:dec}
Verification of \muladom properties over run-bounded \skab s is
decidable,
and can be reduced to finite-state model checking of propositional $\mu$-calculus.
\end{theorem}

The crux of the proof is to show, given a run-bounded \skab\ $\K$, how to
construct an \emph{abstract transition system} $\pcts{\K}$ that satisfies
exactly the same $\muladom$ properties as the original transition system
$\cts{\K}$. This is done by introducing a suitable bisimulation relation, and
defining a construction of $\pcts{\K}$ based on iteratively ``pruning'' those
branches of $\cts{\K}$ that cannot be distinguished by $\muladom$ properties.

In fact, $\pcts{\K}$ is of size exponential in the size of the initial state of
the \skab\ $\K$ and the bound $b$. Hence, considering the complexity of model
checking of $\mu$-calculus on finite-state transition systems
\cite{ClGP99,Stir01}, we obtain that verification is in \EXPTIME.



\section{Repair Semantics for KABs}
\label{sec:KABWithInconsistencyTolerance}

\skab s are defined by taking a radical approach in the management of
inconsistency: simply reject actions that lead to $\T$-inconsistent
ABoxes. However, an inconsistency could be caused by a small portion of the
ABox, making it desirable to \emph{handle} the inconsistency by
allowing the action execution, and taking care of
\emph{repairing} the resulting state so as to restore consistency while
minimizing the information loss. To this aim, we revise the standard
semantics for KABs so as to manage inconsistency, relying on the research on instance-level evolution of knowledge bases
\cite{Wins90,EiGo92,FMKPA08,CKNZ10b}, and, in particular, on the notion of
\emph{ABox repair}, cf.\ \cite{Bert06,LLRRS10}.

In particular, we assume that in this case the system is equipped with a
\emph{repair service} that is executed every time an action changes the content
of the ABox. In this light, a progression step of the KAB is constituted by two
sub-steps: an \emph{action step}, where an executable action with parameters is
chosen and applied over the current ABox, followed by a \emph{repair step},
where the repair service checks whether the resulting state is $\T$-consistent
or not, and, in the negative case, fixes the content of the ABox resulting from
the action step, by applying its repair strategy.

\smallskip
\noindent {\bf Repairing ABoxes.}
We illustrate our approach by considering two specific forms of repair that
have been proposed in the literature \cite{EiGo92} and are applicable to the
context of DL ontologies \cite{LLRRS10}.
\begin{compactitem}
\item Given an ABox $\A$ and a TBox $\T$, a \emph{bold-repair}
  (\emph{b-repair}) of $\A$ with $\T$ is a maximal $\T$-consistent subset $\A'$
  of $\A$.  Clearly, there might be more than one bold-repair for given $\A$
  and $\T$.  By $\Rep(\A,\T)$ we denote the set of \emph{all} b-repairs of $\A$
  with $\T$.
\item A \emph{certain-repair} (\emph{c-repair}) of $\A$ with $\T$ is the ABox
  defined as follows: $\A'=\cap_{\A''\in\Rep(\A,\T)}\A''$.  That is, a c-repair
  of $A$ with $T$ contains only those ABox statements that occur in every
  b-repair of $A$ with $T$.
\end{compactitem}
In general, there are (exponentially) many b-repairs of an ABox
$A$ with $T$, while by definition there is a single c-repair.

\begin{example}\label{Ex-RepairExample}
  \footnotesize{Continuing Example~\ref{ex:exec}, consider the
    $\T$-inconsistent state $\tup{\A',\emptyset}$ obtained from
    \asr{$\action_1()$}{$\gamma_1()$} in \asr{$\initABox$}{$A_0$}.  Its two b-repairs are
    $\Rep(\A',\T) = \{A_1,A_2\}$ with $A_1 = \{C(a)\}$, $A_2 =
    \{D(a)\}$.
    Its c-repair is $\bigcap_{\A \in
      \Rep(\A',\T)} \A = \set{C(a)} \cap \set{D(a)} = \eset$.  \qed}
\end{example}

%
%


\subsection{Bold and Certain Repair Transition Systems}
We now refine the execution semantics of KABs by constructing a
two-layered transition system that reflects the alternation between
the action and the repair steps. In particular, we consider the two
cases for which the repair strategy either follows the bold or certain semantics.
We observe that, if b-repair semantics is applied, then the repair
service has, in general, several possibilities to fix an inconsistent ABox. Since, a-priori, no information about the repair
service can be assumed beside the repair strategy itself, the transition system capturing this execution semantics must
consider the progression of the system for any computable repair,
modelling the repair step as the result of a non-deterministic choice
taken by the repair service when deciding which among the possible
repairs will be the actually enforced one.
This issue does not occur with c-repair semantics, because its repair
strategy is deterministic.

In order to distinguish whether a state
is obtained from an action or repair step,
we introduce a special marker $\temp$,
which is an ABox statement with
a fresh concept name $\mathsf{State}$
and a fresh constant $\mathit{temp}$, s.t.:
if $\temp$ is in the
        current state, this means that the state has been produced by
        an action step, otherwise by the repair step.

Formally, the \emph{b-transition system} $\bs{\K}$
        (resp.~\emph{c-transition system} $\cs{\K}$)
        for a KAB
        \asr{$\K =(\tbox,\initABox,\actSet,\procSet)$}{$\K =(\T,\A_0,\Gamma,\Pi)$}  is a
        (possibly infinite-state)
        transition system
  $(\CONST,\T,\Sigma,s_0,\abox,\trans)$
  where:
  \begin{compactenum}[(1)]
        \item
        \asr{$s_0 = \tup{\initABox,\emptyset}$}{$s_0 = \tup{\A_0,\emptyset}$};
      \item $\Sigma$ and
$\trans$ are defined by simultaneous induction
as the smallest sets satisfying the following properties:
        \begin{compactenum}[(i)] \item $s_0 \in \Sigma$;
        \item (\emph{action step})
                if $\tup{\A,\rmap} \in \Sigma$
                and
                $\temp \not \in \A$,
                then for
                all actions \asr{$\action$}{$\alpha$} in \asr{$\actSet$}{$\Gamma$},
                for all substitutions $\sigma$
                for the parameters
                of \asr{$\action$}{$\alpha$}
                and for all
                $\tup{\A',\rmap'}$
                s.t.~\asr{$\tup{\tup{\A,\rmap},
                \action\sigma,\tup{\A',\rmap'}}
                \in \exec_\K$}{$\tup{\tup{\A,\rmap},
                \alpha\sigma,\tup{\A',\rmap'}}
                \in \exec_\K$},
                let
                $A'' = A' \cup \{\temp\}$,
                and then
                $\tup{\A'',\rmap'}\in\Sigma$
                and
                $\tup{\A,\rmap}\trans \tup{\A'',\rmap'}$;
        \item
        (\emph{repair step})
        if $\tup{\A,\rmap} \in \Sigma$
        and
        $\temp \in \A$,
        then
        for b-repair $A'$
        (resp.~c-repair $A'$)
        of $\A-\{\temp\}$ with $\T$, we have
        $\tup{\A',\rmap}\in\Sigma$ and
        $\tup{\A,\rmap} \trans \tup{\A',\rmap}$.
\end{compactenum}
\end{compactenum}
We refer to KABs with b-transition
(resp.~c-transition) system semantics
as b-KAB (resp.~c-KAB).

\begin{example}\label{Ex-RKABTransitionSystem}
\label{ex:ts-rep}
\footnotesize{ Under b-repair semantics, the KAB in our running
  example looks as follows. Since $A'$ is $T$-inconsistent, we have
  two bold repairs, $A_1$ and $A_2$, which in turn \asr{give rise}{give raise} to two
  runs: $\tup{A_0, \emptyset} \trans \tup{A'_r, \emptyset} \trans
  \tup{A_1, \emptyset}$ and $\tup{A_0, \emptyset} \trans \tup{A'_r,
    \emptyset} \trans \tup{A_2, \emptyset}$, where $A'_r = \{A' \cup
  \{\temp\}$. Since instead \asr{$\action_1$}{$\gamma_1$} does not lead to
  any inconsistency, for every candidate successor $A'' = \{G(x)\}$
  with $m = \{(f(a) \mapsto x)\}$ (see Example~\ref{ex:ts}), we \asd{we} have
  $\tup{A_0, \emptyset} \trans \tup{A'' \cup \{\temp\},
    m} \trans \tup{A'', m}$, reflecting that in this case the repair
  service just maintains the resulting ABox unaltered.  \qed }
\end{example}

\subsection{Verification Under Repair Semantics}
\label{sec:verification2}

We observe that the alternation between an action and a repair step makes EQL queries meaningless for the
intermediate states produced as a result of action steps, because the
resulting ABox could be in fact $\T$-inconsistent (see, e.g.,  state
$\tup{A'_r,\emptyset}$ in Example~\ref{ex:ts-rep}). In fact,
such intermediate states are needed 
just to capture the dynamic structure
that reflects the behaviour of the system. E.g., state
$\tup{A'_r,\emptyset}$ in Example~\ref{ex:ts-rep} has two successor
states, attesting that the repair service with bold semantics will
produce one between two possible repairs.

In this light, we introduce the
\emph{inconsistency-tolerant} temporal logic $\muladomit$,
which is a fragment of $\muladom$ defined as:
\[
  \Phi ~:=~ Q \mid \lnot \Phi \mid \Phi_1 \land \Phi_2
  \mid \exists x.\Phi \mid \DIAM{\BOX{\Phi}} \mid
  \BOX{\BOX{\Phi}} \mid Z \mid \mu Z.\Phi
\]
Beside the standard abbreviations introduced for $\muladom$, we also
make use of the following: $\DIAM{\DIAM{\Phi}} = \neg \BOX{\BOX{\neg
    \Phi}}$, and $\BOX{\DIAM{\Phi}} = \neg \DIAM{\BOX{\neg
    \Phi}}$.
This logic can be used to express interesting properties over b- and c-KABs, exploiting different combinations
of the $\DIAM{}$ and $\BOX{}$ next-state operators so as to quantify over the
possible action steps and corresponding repair steps, ensuring at the
same time that only the $\T$-consistent states produced by the repair
steps are
queried. For example,
$\mu Z.(\Phi \lor \DIAM{\DIAM{Z}})$ models the ``optimistic''
reachability of $\Phi$, stating that there exists a sequence of action
and repair steps, s.t.~$\Phi$
eventually holds. Conversely, $\mu Z.(\Phi \lor \DIAM{\BOX{Z}})$
models the ``robust'' reachability of $\Phi$, stating the existence of
a sequence of action steps leading to $\Phi$ independently from the
behaviour of the repair service. This patterns can be nested into more
complex properties that express requirements about the acceptable
progressions of the system, taking into account data and
repairs. E.g., $\nu Z.(\forall x. Stud(x) \limp \mu Y.(Grad(x) \lor
\DIAM{\BOX{Y}})) \land \BOX{\BOX{Z}}$ states that, for every student
$x$ encountered in any state of the system, it is possible to ``robustly''
reach a state where $x$ becomes graduated.

Since for a given ABox there exist finitely many b-repairs, and one
c-repair, the technique used to prove decidability of properties for
run-bounded \skab s can be extended to deal with b- and c-KABs as well.

\begin{theorem}
\label{thm:itkabs-dec}
Verification of \muladomit properties over run-bounded b-KABs and c-KABs is
decidable.
\end{theorem}

The precise relationship between b-KABs and c-KABs remains to be investigated.

\section{Extended Repair Semantic for KABs}
B-KABs and c-KABs provide an inconsistency-handling
semantics to KABs. However, despite dealing with possible repairs
when some action step produces a $\T$-inconsistent ABox, they do not
explicitly track whether a repair has been actually enforced, nor do they provide 
finer-grained insights about which TBox assertions were
involved in the inconsistency. We extend the repair execution
semantics of so as to equip the transition system with this additional
information. 

While \dllitea does not allow, in general, to uniquely extract from a
$\T$-inconsistent ABox a set of individuals that are responsible for the inconsistency \cite{CDLLR07}, its \emph{separability}
property \cite{CDLLR07} guarantees that inconsistency arises because
a single negative TBox assertion is violated. More specifically, a
$\T$-inconsistency involves the violation of either  a
functionality assertion or negative inclusion in $\T$. 
Since \dllitea obeys to the restriction that no functional
role can be specialized, the first case can be detected by just considering the ABox and the functionality
assertion alone. Contrariwise, the second requires also to take into account
the positive inclusion assertions (since disjointness propagates downward to the
subclasses). Thanks to the FO rewritability of ontology
satisfiability in \dllitea \cite{CDLLR07}, \asr{check}{each such checks} can be done by
constructing a FOL boolean query that corresponds to the
considered functional or negative inclusion assertion, and that can be
directly evaluated over the ABox, considered as a database of facts.

Following \cite{CDLLR07}, given a functionality assertion $\funct{R}$, we
construct the query $\qunsatf(\funct{R}) = \exists x,
x_1,x_2.\eta(R,x,x_1) \land \eta(R,x,x_2) \land x_1 \neq x_2$, where
$\eta(R,x,y) = P(x,y)$ if $R = P$, and $\eta(R,x,y) = P(y,x)$ if $R =
P^-$. Given a negative concept inclusion $B_1 \ISA \neg B_2$ and a set
of positive inclusions $\T_p$, we construct the query $\qunsatn(B_1
\ISA \neg B_2,\T_p) = \rew{\T_p,\exists x. \gamma(B_1,x) \land
  \gamma(B_2,x)}$, where $\gamma(B,x) = N(x)$ if $B = N$,
$\gamma(B,x) = P(x,\_)$ if $B = \exists P$, and $\gamma(B,x) =
P(\_,x)$ if $B = \exists P^-$. 
Similarly, given a negative role
inclusion $R_1 \ISA \neg R_2$, we construct the query  $\qunsatn(R_1
\ISA \neg R_2,\T_p) = \rew{\T_p,\exists x_1,x_2. \eta(R_1,x_1,x_2) \land
  \eta(R_2,x_1,x_2)}$.

\subsection{Extended Repair Transition System}
With this machinery at hand, given a KB $(\T,\A)$ we can now compute the set of TBox
assertions in $\T$ that are actually violated by \asd{a} $\A$. To do so, we
assume wlog that $\iconst$ contains one distinguished constant per
TBox assertion in $\T$, and introduce a function $\lab$, that, given a TBox
assertion, returns the corresponding constant.
We then define the set $\viol(\A,\T)$ of constants labeling TBox
assertions in $\T$ violated by $\A$, as:
\[
\begin{array}{@{}l}
  \{ d \in \Delta \mid \exists t \in \T_f \text{ s.t. }
  d = \lab(t) \text{ and } \A \models \qunsatf(t) \} \cup{}\\
  \{ d \in \Delta \mid \exists t \in \T_n \text{ s.t. }
  d = \lab(t) \text{ and } \A \models \qunsatn(t,\T_p)\}
\end{array}
\]

\begin{example}
Consider $\K$ in Example~\ref{RunningExample}, with $\T = \{C \ISA  \neg
D\}$, and $A' = \{D(a), C(a)\}$ in Example~\ref{ex:exec}. Assume that
$\lab(C \ISA \neg D) = \ell$. We have $\phi = \qunsatn(C
\ISA \neg D,\emptyset) = \exists x. C(x) \land D(x)$. Since $A'
\models \phi$, we get $\viol(\A',\T) = \{\ell\}$.
\end{example}

We now employ this information assuming that the repair service
decorates the states it produces with information about which TBox
functional and negative inclusion
assertions have been involved in the repair. This is done with
a fresh concept $\textsf{Viol}$ that keeps track of the labels of
violated TBox assertions.

Formally, we define the \emph{eb-transition system} $\mf{\K}^{eb}$
(resp.~\emph{ec-transition system} $\mf{\K}^{ec}$) for KAB \asr{$\K
=(\tbox,\initABox,\actSet,\procSet)$}{$\K
=(\T,\A_0,\Gamma,\Pi)$} as a
  	(possibly infinite-state) 
	transition system  $(\CONST,\T,\Sigma,s_0,\abox,\trans)$
        constructed starting from $\mf{\K}^b$
        (resp.~$\mf{\K}^c$) by refining the repair step as follows:
	if $\tup{\A,\rmap} \in \Sigma$ 
	and 
	$\temp \in \A$, 
	then 	
	for b-repair $A'$ 
	(resp.~c-repair $A'$)
	of $\A-\{\temp\}$ with $\T$, we have
	$\tup{\A'_v,\rmap}\in\Sigma$ and 
	$\tup{\A,\rmap} \trans \tup{\A'_v,\rmap}$, where 
$A'_v = A' \cup \{\textsf{Viol}(d) \mid d \in
  \viol(\A',\T)\}$.



\subsection{Verification Under Extended Repair Semantics}
Thanks to the insertion of information about violated TBox assertions
in their transition systems, eb-KABs and ec-KABs support the
verification of $\muladomit$ properties that mix dynamic requirements
with queries over the instance-level information and over the
meta-level information related to inconsistency.  Notice that such
properties can indirectly refer to specific TBox assertions, thanks to the fact that their
labels belong to the set of distinguished constants
$\iconst$. Examples of formulae focused on the presence of violations
in the system are:
\begin{compactitem}
\item $\nu Z.(\neg \exists l. \textsf{Viol}(l)) \land \BOX{\BOX{Z}}$ says that no
  state of the system is manipulated by the repair service;
\item $\nu Z.(\forall l.  \textsf{Viol}(l) \limp (\mu Y. \nu W. \neg
  \textsf{Viol}(l)
  \land \BOX{\BOX{W}} \lor \DIAM{\BOX{Y}}) \land \BOX{\BOX{Z}}$ says
  that, in all states, whenever a TBox assertion $a$ is violated, independently
  from the applied repairs there exists a run that reaches a state
  starting from which $a$ will never be violated anymore.
\end{compactitem}

\label{sec:verification-eit}








Since the TBox assertions are finitely many and fixed for a given KAB,
the key decidability result of Theorem~\ref{thm:itkabs-dec} can be
seamlessly carried over to these extended repair semantics.
\begin{theorem}
\label{thm:eitkabs-dec}
Verification of \muladomit properties over run-bounded eb-KABs and ec-KABs is
decidable.
\end{theorem}

\subsection{From Standard to Extended Repair KABs}
It is clear that extended repair KABs are richer than repair KABs. 
  We now show that eb- and ec-KABs are also richer than \skab s, thanks to the fact that
information about the violated TBox assertions is explicitly tracked
in all states resulting from a repair step. In particular,
verification of \muladom properties over a KAB $\K$ under standard semantics can be recast as a
corresponding verification problem over $\K$ interpreted
either under extended bold or extended certain repair semantics. The intuition behind
the reduction is that a property holds over $\mf{\K}^s$  if that property
holds in the portion of the $\mf{\K}^{eb}$ (or $\mf{\K}^{ec}$) where no TBox
assertion is violated. The absence of violation can be checked over $\T$-consistent states by
issuing the EQL query $\neg \exists
x. [\textsf{Viol}(x)]$. Technically, we define a translation function $\cittransl$
 that transforms an arbitrary \muladom property $\Phi$ into a
 \muladomit property $\Phi' = \cittransl(\Phi)$. The translation $\cittransl(\Phi)$ is inductively defined by recurring over the structure of $\Phi$
   and substituting each occurrence of $\DIAM{\Psi}$ with
   $\DIAM{\DIAM{((\neg \exists x. \textsf{Viol}(x)) \land\cittransl(\Psi))}}$, and
   each  occurrence of $\BOX{\Psi}$ with $\BOX{\DIAM{((\neg \exists
       x. \textsf{Viol}(x)) \limp \cittransl(\Psi))}}$. Observe that,
   in $\tau$, the choice of $\DIAM{}$ for the nested operator can be
   substituted by $\BOX{}$, because for $\T$-consistent
   states produced by an action step, the repair step simply copy the
   resulting state, generating a unique successor even in the eb-semantics.


\begin{theorem} 
Given a KAB $\K$ and a $\muladom$ property $\Phi$, 
$\mf{\K}^s \models \Phi$ iff $\mf{\K}^{eb}
\models \cittransl(\Phi)$ iff $\mf{\K}^{ec}
\models \cittransl(\Phi)$. 
\end{theorem}
The correctness of this result can be directly obtained by considering
the semantics of $\muladom$ and $\muladomit$, and the construction of
the transition systems under the three semantics.

\section{Weakly Acyclic KABs}

So far, all the decidability results here presented have relied on
the assumption that the considered KAB is state-bounded.
As pointed out in \cite{BCDDM13}, \emph{run boundedness} is a semantic
condition that is undecidable to check. In \cite{BCDDM13}, a
sufficient, syntactic condition borrowed from \emph{weak acyclicity}
in data exchange \cite{FKMP05} has been proposed to actually check
whether the KAB under study is run bounded and, in turn, verifiable.

Intuitively, given a KAB $\K$, this test constructs a dependency graph
tracking how the actions of $\K$ transport values from one
state to the next one. To track all the actual dependencies, every
involved query is first rewritten considering the positive inclusion
assertions of the TBox. Two types of dependencies are tracked: copy of
values and use of values as parameters of a service call. $\K$ is
said to be \emph{weakly acyclic} if there is no cyclic chain of dependencies
of the second kind. The presence of such a cycle could produce an
infinite chain of fresh values generation through service calls.

The crux of the proof showing that weakly acyclicity ensures run
boundedness is based on the notion of \emph{positive
  dominant}, which creates a simplified version of the KAB that, from
the execution point of view, obeys to three key properties.
%
First, its execution consists of a single run that closely resembles the chase of a set of tuple-generating dependencies, which terminates
under the assumption of weak acyclicity \cite{FKMP05}, guaranteeing
that the positive dominant is indeed run-bounded. Second, it considers only the positive inclusion assertions of the
TBox, therefore producing always the same behaviour independently from
which execution semantics is chosen, among the ones discussed in this
paper. Third, for every run contained in each of the transition
systems generated under the standard, bold repair, certain repair, and
their extended versions, the values accumulated along the run are
``bounded'' by the ones contained in the unique run of the positive
dominant. This makes it possible to directly carry run-boundedness
from the positive dominant to the original KAB, independently from
which execution semantics is considered.


\begin{theorem}
\label{thm:wa}
Given a weakly acyclic KAB $\K$, we have that $\mf{\K}^s$, $\mf{\K}^b$, $\mf{\K}^c$, $\mf{\K}^{eb}$, $\mf{\K}^{ec}$ 
are all run-bounded. 
\end{theorem}
Theorem~\ref{thm:wa} shows that weak acyclicity is an effective method
to check verifiability of KABs under all inconsistency-aware semantics
considered in this paper.

\section{Conclusion}
\label{sec:conclusion}
We have approached the problem of inconsistency
handling in Knowledge and Action Bases, by resorting to an approach
based on ABox repairs. 
An orthogonal approach to the one taken is to maintain ABoxes
that are inconsistent with the TBox as states of the transition system, and
rely, both for the progression mechanism and for answering queries
used in verification, on consistent query answering
\cite{Bert06,LLRRS10}. Notably, we are able to show that the
decidability and complexity results established for the repair-based
approaches carry over also to this setting. It remains open to
investigate the relationship between these orthogonal approaches to
dealing with inconsistency in KABs.

\bibliographystyle{named}
\bibliography{main-bib}

\begin{thebibliography}{}

\bibitem[\protect\citeauthoryear{Bagheri~Hariri \bgroup \em et al.\egroup
  }{2012}]{BCDD*12}
Babak Bagheri~Hariri, Diego Calvanese, Giuseppe De~Giacomo, Riccardo
  De~Masellis, Marco Montali, and Paolo Felli.
\newblock Verification of description logic {K}nowledge and {A}ction {B}ases.
\newblock In {\em Proc.\ of the 20th Eur.\ Conf.\ on Artificial Intelligence
  (ECAI~2012)}, pages 103--108, 2012.

\bibitem[\protect\citeauthoryear{Bagheri~Hariri \bgroup \em et al.\egroup
  }{2013}]{BCDDM13}
Babak Bagheri~Hariri, Diego Calvanese, Giuseppe De~Giacomo, Alin Deutsch, and
  Marco Montali.
\newblock Verification of relational data-centric dynamic systems with external
  services.
\newblock In {\em Proc.\ of the 32nd ACM SIGACT SIGMOD SIGART Symp.\ on
  Principles of Database Systems (PODS~2013)}, 2013.

\bibitem[\protect\citeauthoryear{Bertossi}{2006}]{Bert06}
Leopoldo~E. Bertossi.
\newblock Consistent query answering in databases.
\newblock {\em {SIGMOD} Record}, 35(2):68--76, 2006.

\bibitem[\protect\citeauthoryear{Burkart \bgroup \em et al.\egroup
  }{2001}]{BCMS01}
O.~Burkart, D.~Caucal, F.~Moller, and B.~Steffen.
\newblock Verification of infinite structures.
\newblock In {\em Handbook of Process Algebra}. Elsevier Science, 2001.

\bibitem[\protect\citeauthoryear{Calvanese \bgroup \em et al.\egroup
  }{2007a}]{CDLLR07b}
Diego Calvanese, Giuseppe De~Giacomo, Domenico Lembo, Maurizio Lenzerini, and
  Riccardo Rosati.
\newblock {EQL-Lite}: {E}ffective first-order query processing in description
  logics.
\newblock In {\em Proc.\ of the 20th Int.\ Joint Conf.\ on Artificial
  Intelligence (IJCAI~2007)}, pages 274--279, 2007.

\bibitem[\protect\citeauthoryear{Calvanese \bgroup \em et al.\egroup
  }{2007b}]{CDLLR07}
Diego Calvanese, Giuseppe De~Giacomo, Domenico Lembo, Maurizio Lenzerini, and
  Riccardo Rosati.
\newblock Tractable reasoning and efficient query answering in description
  logics: The \textit{DL-Lite} family.
\newblock {\em J.\ of Automated Reasoning}, 39(3):385--429, 2007.

\bibitem[\protect\citeauthoryear{Calvanese \bgroup \em et al.\egroup
  }{2009}]{CDLL*09}
Diego Calvanese, Giuseppe De~Giacomo, Domenico Lembo, Maurizio Lenzerini,
  Antonella Poggi, Mariano Rodr{\'\i}guez-Muro, and Riccardo Rosati.
\newblock Ontologies and databases: The \textit{DL-Lite} approach.
\newblock In Sergio Tessaris and Enrico Franconi, editors, {\em Reasoning Web.
  Semantic Technologies for Informations Systems -- 5th Int.\ Summer School
  Tutorial Lectures (RW~2009)}, volume 5689 of {\em Lecture Notes in Computer
  Science}, pages 255--356. Springer, 2009.

\bibitem[\protect\citeauthoryear{Calvanese \bgroup \em et al.\egroup
  }{2010}]{CKNZ10b}
Diego Calvanese, Evgeny Kharlamov, Werner Nutt, and Dmitriy Zheleznyakov.
\newblock Evolution of \textit{DL-Lite} knowledge bases.
\newblock In {\em Proc.\ of the 9th Int.\ Semantic Web Conf.\ (ISWC~2010)},
  volume 6496 of {\em Lecture Notes in Computer Science}, pages 112--128.
  Springer, 2010.

\bibitem[\protect\citeauthoryear{Calvanese \bgroup \em et al.\egroup
  }{2012}]{CDLMS12}
Diego Calvanese, Giuseppe De~Giacomo, Domenico Lembo, Marco Montali, and Ario
  Santoso.
\newblock Ontology-based governance of data-aware processes.
\newblock In {\em Proc.\ of the 6th Int.\ Conf.\ on Web Reasoning and Rule
  Systems (RR~2012)}, volume 7497 of {\em Lecture Notes in Computer Science},
  pages 25--41. Springer, 2012.

\bibitem[\protect\citeauthoryear{Clarke \bgroup \em et al.\egroup
  }{1999}]{ClGP99}
Edmund~M. Clarke, Orna Grumberg, and Doron~A. Peled.
\newblock {\em Model checking}.
\newblock The MIT Press, Cambridge, MA, USA, 1999.

\bibitem[\protect\citeauthoryear{Deutsch \bgroup \em et al.\egroup
  }{2007}]{DeSV07}
Alin Deutsch, Liying Sui, and Victor Vianu.
\newblock Specification and verification of data-driven web applications.
\newblock {\em J.\ of Computer and System Sciences}, 73(3):442--474, 2007.

\bibitem[\protect\citeauthoryear{Deutsch \bgroup \em et al.\egroup
  }{2009}]{DHPV09}
Alin Deutsch, Richard Hull, Fabio Patrizi, and Victor Vianu.
\newblock Automatic verification of data-centric business processes.
\newblock In {\em Proc.\ of the 12th Int.\ Conf.\ on Database Theory
  (ICDT~2009)}, pages 252--267, 2009.

\bibitem[\protect\citeauthoryear{Eiter and Gottlob}{1992}]{EiGo92}
Thomas Eiter and Georg Gottlob.
\newblock On the complexity of propositional knowledge base revision, updates
  and counterfactuals.
\newblock {\em Artificial Intelligence}, 57:227--270, 1992.

\bibitem[\protect\citeauthoryear{Emerson}{1997}]{Emer97}
E.~Allen Emerson.
\newblock Model checking and the {M}u-calculus.
\newblock In N.~Immerman and P.~Kolaitis, editors, {\em Proc.\ of the DIMACS
  Symposium on Descriptive Complexity and Finite Model}, pages 185--214.
  American Mathematical Society Press, 1997.

\bibitem[\protect\citeauthoryear{Fagin \bgroup \em et al.\egroup
  }{2005}]{FKMP05}
Ronald Fagin, Phokion~G. Kolaitis, Ren{\'e}e~J. Miller, and Lucian Popa.
\newblock Data exchange: {S}emantics and query answering.
\newblock {\em Theoretical Computer Science}, 336(1):89--124, 2005.

\bibitem[\protect\citeauthoryear{Flouris \bgroup \em et al.\egroup
  }{2008}]{FMKPA08}
Giorgos Flouris, Dimitris Manakanatas, Haridimos Kondylakis, Dimitris
  Plexousakis, and Grigoris Antoniou.
\newblock Ontology change: Classification and survey.
\newblock {\em Knowledge Engineering Review}, 23(2):117--152, 2008.

\bibitem[\protect\citeauthoryear{Ghallab \bgroup \em et al.\egroup
  }{2004}]{GhNT04}
Malik Ghallab, Dana~S. Nau, and Paolo Traverso.
\newblock {\em Automated planning -- {T}heory and Practice}.
\newblock Elsevier, 2004.

\bibitem[\protect\citeauthoryear{Lembo \bgroup \em et al.\egroup
  }{2010}]{LLRRS10}
Domenico Lembo, Maurizio Lenzerini, Riccardo Rosati, Marco Ruzzi, and
  Domenico~Fabio Savo.
\newblock Inconsistency-tolerant semantics for description logics.
\newblock In {\em Proc.\ of the 4th Int.\ Conf.\ on Web Reasoning and Rule
  Systems (RR~2010)}, pages 103--117, 2010.

\bibitem[\protect\citeauthoryear{Limonad \bgroup \em et al.\egroup
  }{2012}]{LDLHV12}
Lior Limonad, Pieter De~Leenheer, Mark Linehan, Rick Hull, and Roman Vaculin.
\newblock Ontology of dynamic entities.
\newblock In {\em Proc.\ of the 31st Int.\ Conf.\ on Conceptual Modeling
  (ER~2012)}, 2012.

\bibitem[\protect\citeauthoryear{Park}{1976}]{Park76}
David Michael~Ritchie Park.
\newblock Finiteness is {M}u-ineffable.
\newblock {\em Theoretical Computer Science}, 3(2):173--181, 1976.

\bibitem[\protect\citeauthoryear{Poggi \bgroup \em et al.\egroup
  }{2008}]{PLCD*08}
Antonella Poggi, Domenico Lembo, Diego Calvanese, Giuseppe De~Giacomo, Maurizio
  Lenzerini, and Riccardo Rosati.
\newblock Linking data to ontologies.
\newblock {\em J.\ on Data Semantics}, X:133--173, 2008.

\bibitem[\protect\citeauthoryear{Stirling}{2001}]{Stir01}
Colin Stirling.
\newblock {\em Modal and Temporal Properties of Processes}.
\newblock Springer, 2001.

\bibitem[\protect\citeauthoryear{Vianu}{2009}]{Vian09}
Victor Vianu.
\newblock Automatic verification of database-driven systems: a new frontier.
\newblock In {\em Proc.\ of the 12th Int.\ Conf.\ on Database Theory
  (ICDT~2009)}, pages 1--13, 2009.

\bibitem[\protect\citeauthoryear{Winslett}{1990}]{Wins90}
Marianne Winslett.
\newblock {\em {U}pdating {L}ogical {D}atabases}.
\newblock Cambridge University Press, 1990.

\end{thebibliography}

\clearpage
\appendix
\section{Bisimulation and Invariance}

\newcommand{\anseq}{\simeq}

We start by introducing the notion of isomorphism between ABoxes.
Two ABoxes $A_1$ and $A_2$ are \emph{isomorphic}, written $A_1 \equiv A_2$, if there exists a
bijection $h: S_1 \rightarrow S_2$, with $\adom{A_1} \cup \iconst
\subseteq S_1$ and $\adom{A_2} \cup \iconst \subseteq S_2$, which is
the identity on $\iconst$, and s.t.:
\begin{compactenum}
\item for every concept assertion $N(d) \in A_1$, $N(h(d)) \in A_2$;
\item for every role assertion $P(d_1,d_2) \in A_1$, $N(h(d_1),h(d_2))
  \in A_2$;
\item for every concept assertion $N(d) \in A_2$, $N(h^{-1}(d)) \in A_1$;
\item for every role assertion $P(d_1,d_2) \in A_2$, $N(h^{-1}(d_1),h^{-1}(d_2))
  \in A_1$.
\end{compactenum}
We write $A_1 \equiv_h A_2$ to make $h$ explicit. Furthermore, with a
slight abuse of notation, we write $A_2 = h(A_1)$, and $A_1 =
h^{-1}(A_2)$, when there exists a
bijection $h: S_1 \rightarrow S_2$, with $\adom{A_1} \cup \iconst
\subseteq S_1$ and $\adom{A_2} \cup \iconst \subseteq S_2$, s.t.~$A_1
\equiv_h A_2$.

It is easy to see that isomorphism implies the following results.

\begin{lemma}
\label{lemma:iso-query}
Consider two knowledge bases $(T,A_1)$ and $(T,A_2)$, s.t.~there
exists a bijection $h$ with $A_2 = h(A_1)$. For every EQL query $q$,
we have $\tup{d_1,\ldots,d_n} \in \Ans{q,T,A_1}$
  iff $\tup{h(d_1),\ldots,h(d_n)} \in \Ans{h(q),T,h(A_1)}$.
\end{lemma}

\begin{proof}
Trivial, by recalling the notion of first-order rewritability of
EQL queries, and the fact that first-order logic cannot distinguish
between isomorphic structures.
\end{proof}

We now recast the notion of \emph{history preserving bisimulation} as
defined in \cite{BCDDM13} in the context of KABs.
Let $\Upsilon_1 = (\CONST_1,\T,\Sigma_1,s_0,\abox_1,\trans_1)$ and
$\Upsilon_1 = (\CONST_2,\T,\Sigma_2,s_0,\abox_2,\trans_2)$ be transition
systems, with $\abox(s_0) \subseteq \iconst \subseteq \CONST_1 \cap
\CONST_2$. Let $H$ be the set of partial bijections between
$\CONST_1$ and $\CONST_2$, which are the identity over $\iconst$.  A \emph{history preserving bisimulation} between
$\Upsilon_1$ and $\Upsilon_2$ is a relation $\B \subseteq
  \Sigma_1 \times H \times\Sigma_2$ such that $\tup{s_1,h,s_2} \in
  \B$ implies that:
  \begin{compactenum}
   \item $h$ is a partial bijection between $\CONST_1$ and $\CONST_2$, 
     s.t.~$h$ fixes $\iconst$ and $\abox_1(s_1) \equiv_h \abox_2(s_2)$; 
 \item for each $s_1'$, if $s_1 \Rightarrow_1 s_1'$ then there is
    an $s_2'$ with $s_2 \Rightarrow_2 s_2'$ and a bijection
    $h'$ that extends $h$, such that $\tup{s_1',h',s_2'}\in\B$.
  \item for each $s_2'$, if $s_2 \Rightarrow_2 s_2'$ then there is
    an $s_1'$ with $s_1 \Rightarrow_1 s_1'$ and a bijection
    $h'$ that extends $h$, such that $\tup{s_1',h',s_2'}\in\B$.
 \end{compactenum}
A state $s_1 \in \Sigma_1$ is \emph{history preserving bisimilar} to $s_2 \in
  \Sigma_2$ \emph{wrt a partial bijection} $h$, written $s_1 \hbsim_h s_2$, if
  there exists a history preserving bisimulation $\B$ between $\Upsilon_1$ and
  $\Upsilon_2$ such that
  $\tup{s_1,h,s_2}\in\B$.
A state $s_1 \in \Sigma_1$ is \emph{history preserving bisimilar} to $s_2 \in
  \Sigma_2$, written $s_1 \hbsim s_2$, if there exists a partial bijection $h$
  and a history preserving bisimulation $\B$ between $\Upsilon_1$ and
  $\Upsilon_2$ such that
  $\tup{s_1,h,s_2}\in\B$.
A transition system $\Upsilon_1$ is \emph{history preserving bisimilar} to
  $\Upsilon_2$, written $\Upsilon_1 \hbsim \Upsilon_2$, if there exists a partial
  bijection $h_0$ and a history preserving bisimulation $\B$ between $\Upsilon_1$
  and $\Upsilon_2$ such that $\tup{s_{01},h_0,s_{02}}\in\B$.

The following fundamental results connects history preserving
bisimulation and the logic $\muladom$:
 \begin{theorem}
\label{thm:hbsim}
 Consider two transition systems $\Upsilon_1$ and $\Upsilon_2$ such that
 $\Upsilon_1 \hbsim \Upsilon_2$.
 For every $\muladom$ closed formula $\Phi$, we have:
 $\Upsilon_1 \models \Phi \textrm{ if and only if } \Upsilon_2 \models \Phi$.
 \end{theorem}
\begin{proof}
The proof follows from that of Theorem 3.1 in \cite{BCDDM13},
noticing that, by Lemma~\ref{lemma:iso-query}, isomorphism indeed
preserves certain answers.
\end{proof}

\section{Standard KABs}

\subsection{Proof of Theorem~\ref{thm:dec}}
\label{sec:dec}
In principle, decidability can be obtained by taking advantage from
first-order rewritability of \dllitea, and translating a KAB into a
corresponding Data-Centric Dynamic System \cite{BCDDM13}. However, in
order to make the proof adaptable to the inconsistency-aware semantics
discussed in the paper, we reconstruct the proof contained in
\cite{BCDDM13} over KABs. We first discuss the intuition behind the
proof, and then focus on the technical development.

Given a run-bounded \skab\ $\K$, the crux of the proof is to show how to construct an
\emph{abstract transition system} $\pcts{\K}$ that satisfies exactly
the same $\muladom$ properties of the original transition system
$\cts{\K}$. To do so, a first observation is that the only source of
infiniteness in $\cts{\K}$ is the infinite branching arising when a
service call is issued for the first time. In this case, given a state
$s = \tup{A,\rmap}$ in $\cts{\K}$, for every executable action with legal parameters
$\alpha\sigma$, $s$ contains an infinite number of successor states,
each one corresponding to an assignment of all the newly introduced service
calls to values in $\CONST$, s.t.~the resulting state does not violate
any axiom of $\T$. 

One can see these successors as variations of a
finite set of structures, each one expressing an isomorphic type
(called \emph{equality commitment}) constructed over the set of facts
$E = \doo{}{}{\T,\A,\alpha\sigma}$ and the map $\rmap$,
by fixing the set of equalities and inequalities between the service
calls that must be issued, and the service calls and values contained
in $E$, $\rmap$ and $\iconst$. Each structure can be concretized into a
successor state by evaluating the service calls so as to satisfy the
equalities and inequalities induced by the equality commitment (this
also guarantees that the evaluation agrees with $\rmap$). Two
concretizations of the same structure are isomorphic, i.e., they
contain the same ABox and service call map modulo renaming
of the newly introduced values.

 We now observe that EQL-queries do not
distinguish isomorphic ABoxes. In particular, consider two ABoxes $A_1$ and $A_2$,
and a bijection $h$ that induces an isomorphism between $A_1$
and $A_2$. Now consider an EQL query $q$ s.t.~the constants used in
$q$ appear in $h$, and let $h(q)$ be the query obtained by replacing
such constants through the application of $h$. It is easy to see that
the certain answers of $q$ over $A_1$ are exactly the same of $h(q)$
over $A_2$, modulo renaming of the values via $h$. The key consequence of
this property is that, given a state $s$ of $\cts{\K}$, $\muladom$ is
not able to distinguish successors of $s$ that concretize $E$ by
satisfying the same equality commitment. Therefore, all such
successors can be collapsed into a unique representative successor,
without affecting the satisfaction of a closed $\muladom$ property $\Phi$
asked in the initial state of the system.

By inductively applying this pruning, we can construct a finite-state
transition system $\pcts{\Phi}$. Since the active domain of
$\pcts{\Phi}$ is finite, by quantifier elimination we can then transform $\Phi$
into a corresponding propositional $\mu$-calculus property $\phi$, and reduce
verification of $\Phi$ over $\cts{\K}$ as standard model checking of
$\phi$ over $\pcts{\K}$, which is indeed decidable \cite{Emer97}.


\smallskip 
\noindent
\textbf{Equality commitments.}
Given a set $S \subseteq \scset
\cup \CONST$ containing individuals and service calls, an
\emph{equality commitment} over $S$ is a partition $H$ of $S$
s.t.~every cell of $H$ contains at most one
element $d\in \CONST$. Given an element $e \in S$, we use $[e]_H$ do
denote the cell $e$ belongs to. With a slight abuse of notation, we
say that $e \in H$ if $e \in S$.
Now consider a KAB $\K=(\T,A_0,\Gamma,\Pi)$, a state $\tup{\A,\rmap}$, and an
action $\alpha \in \Gamma$ with parameters $\sigma$,
s.t.~$\alpha\sigma$ is legal in $\tup{\A,\rmap}$ according to $\Pi$. 
Let $\H(\T,\tup{\A,\rmap},\alpha\sigma)$ be the set of equality
commitments $H_i$ constructed over $\adom{\iconst} \cup \adom{A} \cup \domain{\rmap} \cup \image{\rmap} \cup
\adom{\doo{}{}{\T,\A, \alpha\sigma}}$ that \emph{agrees} with $\rmap$,
i.e., for every assignment $(f \ra d)$ in $\rmap$, $[f]_{H_i} = [d]_{H_i}$. Intuitively,
the elements of $\H$ are equality commitments that fix the equivalence
class to which every new service call, introduced by $\doo{}{}{\T,\A,
  \alpha\sigma}$, belongs to.

We say that $\groundexec{}{\T,\A,\alpha\sigma}$ \emph{respects} an
equality commitment $H \in \H(\T,\tup{\A,\rmap},\alpha\sigma)$ if, for
every pair of assignments $(f_1 \ra d_1),(f_2 \ra d_2)$ in 
$\groundexec{}{\T,\A,\alpha\sigma}$, $d_1 = d_2$ iff $f_1$ and $f_2$
belong to the same cell $P$ of $H$, and $d_1 = d_2 = d$ iff $d$
belongs to $P$.




\smallskip
\noindent
\textbf{Pruning.} 
Given a KAB $\K =(\T,A_0,\Gamma,\Pi)$, we refine the definition of
$\exec_\K$ so as to create a parsimonious version that minimally
covers, state-by-state, the various equality commitments. 

\label{pexec}
In particular, we define a transition relation $\pexec_\K$ as follows.
For every $\tup{\tup{\A,\rmap},\alpha\sigma,\tup{\A',\rmap'}} \in
\exec_\K$, fix $\theta = \rmap' \setminus \rmap$ and $H \in
\H(\T,\tup{\A,\rmap},\alpha\sigma)$ s.t.~$\theta$ respects $H$. Then
there exists \emph{only one} $\theta_p =
\groundexec{}{\T,\A,\alpha\sigma}$ s.t.~$\theta_p$ respects $H$ and,
given, $\A_p = \doo{}{}{\T, \A, \alpha\sigma}\theta_p$ and $\rmap' =
\rmap \cup \theta_p$, $\tup{\tup{\A,\rmap},\alpha\sigma,\tup{\A_p,\rmap_p}} \in
\pexec_\K$.
Intuitively, $\pexec_\K$ ``prunes'' $\exec_\K$ by collapsing into a
unique representative tuple all
transitions that are associated to a given starting state and action with
parameters, and that respect the same equality commitment. 

Starting from $\pexec_\K$, we define a \emph{pruning}  $\pcts{\K}$ of the
transition system under standard semantics $\cts{\K}$ as a transition
system constructed following the standard semantics, but by using
$\pexec_\K$ in place of $\exec_\K$ to inductively construct the set of
states and transitions. In general, there exist infinitely many
prunings, whose difference relies in the particular choice for the
representatives when constructing $\pexec_\K$. However, we show that
all such prunings are history-preserving bisimilar to the original
transition system $\sts{\K}$. The following lemma is a key result in
this direction, and intuitively shows that bisimulation does not
distinguish different progressions that fix, step-by-step, the same
equality commitments. In the lemma, for the sake of readability, given a
service call map $\rmap$ and a function $h: \CONST \ra \CONST$ defined
over all values contained in $\rmap_1$ (considering both the service
call parameters and their results), we write $\rmap_2 = f(\rmap_1)$ to
denote the service call map constructed as follows: for every
assignment $(f(d_1,\ldots,d_n) \ra d)$ in $\rmap_1$, we have
$(f(h(d_1),\ldots,h(d_n)) \ra h(d))$ in $\rmap_2$.

\begin{lemma}
\label{lemma:pruning-state}
Let $\K$ be a \skab\ with transition system $\cts{\K}$, and let 
$\pcts{\K}$ be a pruning of $\cts{\K}$. Consider a state $\tup{A,\rmap}$ of
$\cts{\K}$ and a state $\tup{A_p,\rmap_p}$ of $\pcts{\K}$. If there
exists a bijection $h$ s.t.~$A_p = h(A)$ and $\rmap_p = h(\rmap)$ (or, equivalently, $\rmap =
  h^{-1}(\rmap_p)$),
then $\tup{A,\rmap} \hbsim_h \tup{A_p,\rmap_p}$.
\end{lemma}
\begin{proof}
Let $\K =(\T,\A_0,\Gamma,\Pi) $, $\cts{\K} = (\CONST,\T,\Sigma,s_0,\abox,\trans)$, and $\pcts{\K} = (\CONST,\T,\Sigma_p,s_0,\abox,\trans_p)$.
To prove the lemma, we show that, for every state $\tup{A',\rmap'}$ s.t.~$\tup{A,\rmap} \trans
\tup{A',\rmap'}$, there exists a state $\tup{A'_p,\rmap'_p}$ and a
bijection $h'$ s.t.:
\begin{inparaenum}
\item $\tup{A_p,\rmap_p} \trans_p
\tup{A'_p,\rmap'_p}$;
\item $h'$ extends $h$;
\item $A'_p = h'(A')$;
\item $\rmap'_p = h'(\rmap')$.
\end{inparaenum}
By definition of $\cts{K}$, if $\tup{A,\rmap} \trans
\tup{A',\rmap'}$, then there exists an action $\alpha \in \Gamma$ with
parameters $\sigma$ s.t.~$\sigma$ is legal in $\tup{A,\rmap}$
according to $\Pi$, and $\theta \in \groundexec{}{\T,\A,\alpha\sigma}$
s.t.~$\theta$ agrees with $\rmap$, $A' = \doo{}{}{\T, \A,
  \alpha\sigma}\theta$, and $\rmap' = \rmap \cup \theta$. From this
information, we can extract the equality commitment $H \in
\H(\T,\tup{\A,\rmap},\alpha\sigma)$ s.t.~$\theta$ respects $H$.

Since $A_p = h(A)$, from Lemma~\ref{lemma:iso-query} we know that the certain answers
computed over $A$ are the same, modulo renaming through $h$, to those
computed over $A_p$.  Furthermore, since $\sigma$ maps
parameters of $\alpha$ to values in $\adom{A}$, we can construct
$\sigma_p$ mapping parameters of $\alpha$ to values in $\adom{A_p}$, so
as $(x \ra d)$ in $\sigma$ implies $(x \ra h(d))$ in $\sigma_p$.
By hypothesis, we also know that $\rmap_p = h(\rmap)$. As a
consequence, we have that $\sigma_p$ is legal in $\tup{A_p,\rmap_p}$
according to $\Pi$, and that $\H(\T,\tup{\A_p,\rmap_p},\alpha\sigma_p)$
contains the same equality commitments in
$\H(\T,\tup{\A,\rmap},\alpha\sigma)$ up to renaming of individuals
through $h$. Now pick commitment $H_p \in
\H(\T,\tup{\A_p,\rmap_p},\alpha\sigma_p)$ so that $H_p$ corresponds to
$H$ up to renaming of individuals through $h$.

By definition of pruning, we know that there exists a unique
$\theta_p$ that respects $H_p$ (and, in turn, agrees with $m_p$) s.t., given $A'_p = \doo{}{}{\T, \A_p,
  \alpha\sigma_p}\theta_p$ and $\rmap'_p = \rmap_p \cup \theta_p$, we have
$\tup{A_p,\rmap_p} \trans_p \tup{A'_p,\rmap'_p}$. Since $H_p$ corresponds to
$H$ up to renaming of individuals through $h$, $\theta$ respects $H$,
and $\theta_p$ respects $H_p$, we can lift $h$ to an extended
bijection $h'$ s.t.~$\theta_p = h(\theta)$. By construction, this means
that $A'_p = h'(A')$, and that $\rmap'_p = h'(\rmap')$, hence the
claim is proven.

The other direction can be proven in the symmetric way.
\end{proof}

\begin{lemma}
\label{lemma:pruning}
For every \skab\ $\K$ with transition system $\cts{\K}$ and every
pruning $\pcts{\K}$ of $\cts{\K}$, we have $\pcts{\K}\hbsim\cts{\K}$.
\end{lemma}
\begin{proof}
Immediate consequence of Lemma~\ref{lemma:pruning-state}, by noticing
that the initial states of $\cts{\K}$ and $\pcts{\K}$ are the same,
and can be therefore connected through the identity bijection between
their active domains.
\end{proof}

\noindent
\textbf{Proof of Theorem~\ref{thm:dec}.} Given a run-bounded KAB $\K$,
we observe that each pruning $\pcts{\K}$ of $\cts{\K}$ is
finite-state. On the one hand, thanks to run-boundedness each run consists of a
finite number of states. On the other hand, thanks to the definition
of pruning, each state has only finitely many successors. In fact,
given a state of $\pcts{\K}$, there are only finitely many equality
commitments that can be created by considering all possible actions
with parameters. This implies that the entire active domain
$\adom{\pcts{K}}$ of $\pcts{\K}$ is finite as well.
By Lemma~\ref{lemma:pruning} and Theorem~\ref{thm:hbsim}, we know that
$\pcts{\K}$ is a faithful abstraction of $\cts{\K}$, i.e., for every
$\muladom$ formula $\Phi$, $\cts{\K}\models \Phi$ iff
$\pcts{\K}\models\Phi$. Taking advantage from the finiteness of
$\adom{\pcts{K}}$, by quantifier elimination we can construct a
propositional $\mu$-calculus property $\phi$
s.t.~$\pcts{\K}\models\Phi$ iff $\pcts{\K}\models\phi$. The proof
completes by observing that verifying whether $\pcts{\K}\models\phi$
amounts to standard model checking of propositional $\mu$-calculus
over finite-state transition systems, which is indeed decidable \cite{Emer97}.

\section{KABs Under Repair Semantics}
We open this section by observing that the repair service does not
distinguish between isomorphic ABoxes.

\begin{lemma}
\label{lemma:iso-rep}
Consider two knowledge bases $(T,A_1)$ and $(T,A_2)$, s.t.~there
exists a bijection $h$ with $A_2 = h(A_1)$. Then for every ABox
$A_1^r$ s.t.~$A_1^r \in \Rep(\A_1,\T)$, we have $h(A_2^r) \in
\Rep(\A_2,\T)$, and for every ABox $A_2^r$ s.t.~$A_2^r \in \Rep(\A_2,\T)$, we have $h^{-1}(A_2^r) \in
\Rep(\A_1,\T)$.
\end{lemma}

\begin{proof}
Trivial, by recalling the notion of first-order rewritability of
ontology satisfiability in \dllitea, and the fact that first-order logic cannot distinguish
between isomorphic structures.
\end{proof}

\subsection{Proof of Theorem~\ref{thm:itkabs-dec}}
\label{sec:itkabs-dec}
Given a $\K$, we introduce the \emph{pruning}  $\Theta_\K$ of the
transition system under repair semantics (denoted by $\mf{\K}^b$ for
the bold semantics, and $\mf{\K}^c$ for the certain semantics), as the transition
system constructed following one between the two repair semantics, but by relying on
the transition relation $\pexec_\K$ (as defined in Section~\ref{pexec}) in
place of $\exec_\K$. Differently from the standard case, to show that
$\Theta_\K \hbsim \mf{\K}^b$ ($\Theta_\K \hbsim \mf{\K}^c$ resp.) we have to deal with the action and
repair step. In particular, we reconstruct
Lemma~\ref{lemma:pruning-state} in this two-steps setting.

\begin{lemma}
\label{lemma:pruning-state-it}
Let $\K$ be a b-KAB (c-KAB respectively) with transition system
$\mf{\K}^b$ ($\mf{\K}^c$ resp.), and let 
$\Theta_\K$ be a pruning of $\mf{\K}^b$ ($\mf{\K}^c$ resp.). Consider a state $\tup{A,\rmap}$ of
$\mf{\K}^b$ ($\mf{\K}^c$ resp.), and a state $\tup{A_p,\rmap_p}$ of $\Theta_\K$. If there
exists a bijection $h$ s.t.~$A_p = h(A)$ and $\rmap_p = h(\rmap)$ (or, equivalently, $\rmap =
  h^{-1}(\rmap_p)$),
then $\tup{A,\rmap} \hbsim_h \tup{A_p,\rmap_p}$.
\end{lemma}
\begin{proof}
Let $\K\!=\!(\T,\A_0,\Gamma,\Pi) $, \mbox{$\mf{\K}^b =
(\CONST,\T,\Sigma,s_0,\abox,\trans)$} (resp., $\mf{\K}^c =
(\CONST,\T,\Sigma,s_0,\abox,\trans)$), and $\Theta_{\K} = (\CONST,\T,\Sigma_p,s_0,\abox,\trans_p)$.
To prove the lemma, we show that, for every state $\tup{A',\rmap'}$ s.t.~$\tup{A,\rmap} \trans
\tup{A',\rmap'}$, there exists a state $\tup{A'_p,\rmap'_p}$ and a
bijection $h'$ s.t.:
\begin{inparaenum}
\item $\tup{A_p,\rmap_p} \trans_p
\tup{A'_p,\rmap'_p}$;
\item $h'$ extends $h$;
\item $A'_p = h'(A')$;
\item $\rmap'_p = h'(\rmap')$.
\end{inparaenum}
To show the claim, we have to separately discuss the case in which
$\temp \not\in A$, and the case in which
$\temp \in A$. The first case is equivalent for
$\mf{\K}^b$ and $\mf{\K}^c$, whereas the second case is different,
since the two semantics diverge when it comes to the repair step
(b-KABs nondeterministically produce one among the possible repairs,
while c-KABs construct a unique repair corresponding to the
intersection of possible repairs). 

\smallskip
\noindent
\textbf{Base case:} trivial, because the transition system and its
pruning start from the same intial state $\tup{A_0,\emptyset}$.

\smallskip
\noindent
\textbf{Case 1 (action step):} $\temp \not\in A$. First of all, we
observe that $\mathit{temp}$ is a distinguished constant of $\iconst$, hence $h(rep)
= rep$. Since $A \equiv_h
A_p$, $\temp \not\in A_p$.
The claim can be then proven
exactly in the same way as done for Lemma~\ref{lemma:pruning-state},
noticing however that each ABox $A'$ s.t.~$A \trans A'$ contains
$\temp$, making the induction hypothesis for case 1
inapplicable, and the one for case 2 applicable.

\smallskip
\noindent
\textbf{Case 2 (repair step) - bold semantics:} $\temp \in A$. By hypothesis, $A_p =
h(A)$, and since $h(rep)$, $\temp \in A_p$ as
well. Notice that $h$ is syntactically applied over the ABoxes $A$
and $A_p$ without involving the TBox $T$, and therefore it can be applied
also when such ABoxes are $\T$-inconsistent.
On the one hand, by construction of the transition system under the bold repair semantics, we
therefore know that:
\begin{compactenum}
\item for every $s'$ s.t.~$\tup{A,\rmap} \trans s'$, we have
  $s'=\tup{A',\rmap}$, with $A' \in \Rep(A-\{\temp\},T)$;
\item for every $s'_p$ s.t.~$\tup{A_p,\rmap_p} \trans_p s'_p$, we have
  $s'_p=\tup{A'_p,\rmap_p}=\tup{A'_p,h(\rmap)}$, with $A'_p \in \Rep(A_p-\{\temp\},T)$.
\end{compactenum}
On the other hand, since $A_p = h(A)$, from Lemma~\ref{lemma:iso-rep}
we get that for every $A'' \in \Rep(A-\{\temp\},T)$, $h(A'') \in
\Rep(A_p-\{\temp\},T)$. 
We therefore obtain that, for every state $\tup{A',\rmap}$
s.t.~$\tup{A,\rmap} \trans \tup{A',\rmap}$, we have $\tup{A_p,\rmap_p}
\trans_p \tup{h(A'),\rmap_p} = \tup{h(A'),h(\rmap)}$.

Finally, notice that, by construction $A'$ and $A'_p$ do not contain
$\temp$. The claim is therefore proven by inductively
applying Case 1 over $A'$, $A'_p$, and $h$.

The other direction can be proven in the symmetric way.

\smallskip
\noindent
\textbf{Case 2 (repair step) - certain semantics:} $\temp \in A$. By hypothesis, $A_p =
h(A)$, and since $h(rep)$, $\temp \in A_p$ as
well. Notice that $h$ is syntactically applied over the ABoxes $A$
and $A_p$ without involving the TBox $T$, and therefore it can be applied
also when such ABoxes are $\T$-inconsistent.
On the one hand, by construction of the transition system under the
certain repair semantics, we
therefore know that:
\begin{compactenum}
\item there exists exactly one $s' = \tup{A',\rmap}$
  s.t.~$\tup{A,\rmap} \trans s'$, where $A' = \bigcap_{A^r \in \Rep(A-\{\temp\},T)} A^r$;
\item there exists exactly one $s'_p 
  =\tup{A'_p,\rmap_p}=\tup{A'_p,h(\rmap)}$ s.t.~$\tup{A_p,\rmap_p}
  \trans_p s'_p$, where $A'_p =
  \bigcap_{A_p^r \in \Rep(A_p-\{\temp\},T)} A_p^r$.
\end{compactenum}
On the other hand, since $A_p = h(A)$, from Lemma~\ref{lemma:iso-rep}
we get that  $A^r \in \Rep(A-\{\temp\},T)$ iff $h(A^r) \in
\Rep(A_p-\{\temp\},T)$. As a consequence, $A'_p =
  \bigcap_{A^r\in \Rep(A-\{\temp\},T)} h(A^r) =
  h(\bigcap_{A^r\in \Rep(A-\{\temp\},T)} A^r) = h(A')$.
Finally, notice that, by construction $A'$ and $A'_p$ do not contain
$\temp$. The claim is therefore proven by inductively
applying Case 1 over $A'$, $A'_p$, and $h$.
\end{proof}

With Lemma~\ref{lemma:pruning-state-it} at hand, we can easily reconstruct
the proof of Theorem~\ref{thm:dec} (given in Section~\ref{sec:dec})
for b- and c-KABs. Since $\muladomit$ is a fragment of \muladom, we get the result.

\section{KABs under Extended Repair Semantics}

\subsection{Proof of Theorem~\ref{thm:eitkabs-dec}}

Given an eb-KAB (ec-KAB respectively) $\K$, we introduce the \emph{pruning}  $\Theta_{\K}$ of the
transition system $\mf{\K}^{eb}$ ($\mf{\K}^{ec}$ resp.), as the transition
system constructed following the extended bold (extended certain,
resp.) repair semantics, but by relying on
the transition relation $\pexec_\K$ (as defined in Section~\ref{pexec}) in
place of $\exec_\K$. To prove  
$\Theta_{\K} \hbsim \mf{\K}^{eb}$ ($\Theta_{\K} \hbsim \mf{\K}^{ec}$ resp.), one can follow step by step the line of
reasoning of Section~\ref{sec:itkabs-dec}, taking into consideration
the fact that $\textsf{Viol}$ concept assertions are inserted into the
ABoxes produced by a repair step. It can be easily noticed that such
assertions do not introduce any additional complication. Remember, in fact, that given an ABox $A$, these
assertions are produced by computing the set $\viol(A,T)$, which is in
turn produced by issuing a series of closed first-order queries over
$A$, considered as a database of facts. Consequently, given two ABoxes $A$ and $A_p$ and a bijection
$h$ s.t.~$A_p = h(A)$, $\viol(A,T) = \viol(h(A),T) =
\viol(A_p,T)$.


\section{Weakly Acyclic KABs}
\label{sec:wa-kabs}
Weakly acyclic KABs are inspired by weakly
acyclic tuple-generating dependencies in data exchange
\cite{FKMP05}. As in data exchange, in our setting weak acyclicity
is a property defined over a \emph{dependency graph}, constructed from
the KAB's specification. In particular, the dependency graph captures the transfer of individuals
from one state to the next state, differentiating between the case of
copy, and the case of service calls. In fact, the latter case leads to
possibly introduce fresh values into the system. The dependency graph
is defined as a variation of the definitions given in \cite{BCDDM13}
and \cite{BCDD*12}.

 Given a KAB $\K=(\T,A_0,\Gamma,\Pi)$, we define its \emph{dependency
    graph} $G = \tup{V, E}$ as follows:
  \begin{compactenum}
  \item Nodes are defined starting from $\T$. More specifically, we
    have one node $\tup{N,1} \in V$ for each concept $N$ in
    $\T$, and two nodes $\tup{P,1},\tup{P,2} \in V$ for every role $P$ in $\T$ (reflecting the
    fact that roles are binary relations, i.e., have two components).
  \item Edges are defined starting from the effect specifications
    in $\Gamma$. We discuss the case of two concept assertions, but In particular:
\begin{compactenum}
\item an ordinary edge $\tup{N_1,1} \ra \tup{N_2,1}$ is contained in $E$ if there
  exists an action $\gamma \in \Gamma$, an effect specification
    \[
    \map{[q^+] \wedge Q^-}{A'}
    \]
in $\gamma$, and a variable or parameter $x$ s.t.~$N_1(x)$ appears in 
$\rew{q^+,\T}$ (i.e., in the perfect rewriting of $q^+$ w.r.t.~$\T$), and
$N_2(x)$ appears in $A'$ (similarly for nodes corresponding to role
assertions). 
\item a special edge $\tup{N_1,1} \xrightarrow{*} \tup{N_2,1}$ is contained in $E$ if there
  exists an action $\gamma \in \Gamma$, an effect specification
    \[
    \map{[q^+] \wedge Q^-}{A'}
    \]
in $\gamma$, and a variable or parameter $x$ s.t.~$N_1(x)$ appears in 
$\rew{q^+,\T}$, and
$N_2(f(\ldots,x,\ldots))$ appears in $A'$ (similarly for nodes corresponding to role
assertions). 
\end{compactenum}
\end{compactenum}
A KAB $\K$ is \emph{weakly
    acyclic} if its dependency graph has no cycle going through a
  special edge.

\subsection{Proof of Theorem \ref{thm:wa}}
To prove the theorem, we resort to the approach discussed in \cite{BCDDM13}
and \cite{BCDD*12}, adapting it so as to deal with inconsistency. More
specifically, the main steps to prove the results are as follows:
\begin{compactenum}
\item Given a KAB $\K$, we introduce its \emph{consistent approximant}
  $\K^p$ and positive dominant $\K^+$,
  which incrementally simplify $\K$ while maintaining the same dependency graph.
\item We show that when $\K$ is weakly acyclic, then it is
  run-bounded.
\item We show that $\K^+$ ``dominates'' $\K^p$ under all semantics
  discussed in the paper, i.e., the active domain of the transition
  system for $\K$ is always
  contained in the active domain of the transition system for $\K^+$.
\item We do the same for $\K$ w.r.t.~$\K^p$, thus transferring the
  weak acyclicity  property from $\K^+$ to $\K$.
\end{compactenum}
Technically, given a KAB $\K=(\T,A_0,\Gamma,\Pi)$, we define the 
\emph{consistent approximant} $\K^p$ of $\K$ as a KAB $ = (T_p, A_0^p, \Gamma^p,
  \Pi)$, where $A_0^p$ and $\Gamma^p$ are obtained as follows:
\begin{compactitem}
\item $A_0^p = A_0 \cup \{\textsf{Viol}(d) \mid \exists t \in T_n \cup
  T_f \text{ s.t. } d=\lab(t)\}$; i.e., $A_0^p$ saturates $A_0$ with
  all possible violations of negative inclusion and functionality
  assertions in $\T$.
\item For every action 
$\act(p_1,\ldots,p_n) : \{e_1, \ldots, e_m\} \in \Gamma$
we have $\act(p_1,\ldots,p_n) : \{e_v,e_1, \ldots, e_m\} \in \Gamma^p$,
where $e_v = \map{\textsf{Viol}(x)}{\{\textsf{Viol}(x)\}}$ copies all
$\textsf{Viol}$ assertions.
\end{compactitem} 
Notice that the TBox of the consistent approximant is constituted by
the positive inclusion assertions of the original TBox. 

Starting from the consistent approximant, we define the 
\emph{positive dominant} $\K^+$ of $\K$ as a KAB $ = (\T_p, A_0^p, \Gamma^+,
  \Pi^+)$, where $\Gamma^+$ and $\Pi^+$ are obtained as follows:
\begin{compactitem}
\item For each action 
$\act(p_1,\ldots,p_n) : \{e_1, \ldots, e_m\} \in \Gamma^p$
we have $\act^+() : \{e_1^+, \ldots, e_m^+\} \in \Gamma^+$ 
where, given $e_i = \map{[q_i^+] \land Q^-}{A_i'}$, we have
$e_i^+ = \map{[q_i^+]}{A_i'}$.
\item For each condition-action rule $\carule{Q}{\act(p_1,\ldots,p_n)} \in \Pi$, we
  have $\carule{true}{\act^+()} \in \Pi^+$. 
\end{compactitem} 
It is easy to show that the dependency graphs of $\K$, $\K^p$ and
$\K^+$ coincide, and therefore $\K$ is weakly acyclic iff $\K^p$ is
weakly acyclic iff $\K^+$ is weakly acyclic.

\begin{theorem}\label{thm:wa-dominant}
  Given KAB $\K$, if $\K$ is weakly acyclic then
  its positive dominant $\K^+$ is run-bounded.
\end{theorem}
\begin{proof}
By compiling away the TBox of $\K^+$ exploiting the first-order
rewritability of \dllitea, the obtained KAB exactly corresponds to the notion
of \emph{positive approximant} defined for relational Data-Centric
Dynamic Systems in \cite{BCDDM13}. 
The proof is then directly obtained from the proof of Theorem 4.7 in \cite{BCDDM13}.
\end{proof}

\noindent
To show that Theorem \ref{thm:wa-dominant} extends to the KAB itself
under each of the semantics considered in this paper, we first
introduce the notion of \emph{dominance} between transition
systems. Technically, a transition system $\Upsilon_1$ is
\emph{dominated by}
$\Upsilon_2$ if, for every run $\tau_1$ in $\Upsilon_1$ there exists a
run $\tau_2$ in $\Upsilon_2$ s.t.~for all pairs of states $\tau_1(i)$
and $\tau_2(i)$, we have $\abox(\tau_1(i)) \subseteq
\abox(\tau_2(i))$. By definition, we consequently have that if
$\Upsilon_2$ is run-bounded, then $\Upsilon_1$ is run-bounded as
well. This shows that, to prove run-boundedness of a transition
system, it is sufficient to prove that such a transition system is dominated by a
run-bounded transition system.


With this machinery at hand, we are now able to prove the following
two key lemmas, which respectively show that for any semantics
considered in this paper, the consistent
approximant is dominated by the positive
dominant, and dominates the original
KAB.

\begin{lemma}\label{PathDomination}
For any KAB $\K$, we have that:
\begin{compactenum}
\item $\mf{\K^p}^s$ is dominated by $\mf{\K^+}^s$;
\item $\mf{\K^p}^b$ is dominated by $\mf{\K^+}^b$;
\item $\mf{\K^p}^c$ is dominated by $\mf{\K^+}^c$;
\item $\mf{\K^p}^{eb}$ is dominated by $\mf{\K^+}^{eb}$;
\item $\mf{\K^p}^{ec}$ is dominated by $\mf{\K^+}^{ec}$.
\end{compactenum}
\end{lemma}

\begin{proof}
We discuss claim 1 and claims 2-5 separately.

\smallskip
\noindent
Each claim can be obtained by proving the following stronger claim:
for every run $\tau$ in
  $\mf{\K^p}^s$ (resp., $\mf{\K^p}^b$, $\mf{\K^p}^c$,
  $\mf{\K^p}^{eb}$, $\mf{\K^p}^{ec}$), there exists a run $\tau^+$ in $\mf{\K^+}^s$ (resp., $\mf{\K^+}^b$, $\mf{\K^+}^c$,
  $\mf{\K^+}^{eb}$, $\mf{\K^+}^{ec}$) s.t.~for
  all pairs of state $\tau(i) = \tup{A_i, \rmap_i}$ and $\tau^+(i) =
  \tup{A^+_i, \rmap^+_i}$, we have:
  \begin{enumerate}
 \item $A_i \subseteq A^+_i$;
  \item $\rmap^+_i$ extends $\rmap_i$;
  \item for the mappings mentioned in $\rmap^+_i$ but not in
    $\rmap_i$, $\rmap_i^+$ ``agrees'' with the maps contained in the
    suffix of $\tau(i)$, i.e.,
    \[\restrict{\pos{\rmap_i}}{C_i} =\restrict{(\bigcup_{j>i}
      \rmap_j)}{C_i}\] where $C_i = \domain{\pos{\rmap_i}} \cap
    \bigcup_{j>i} \domain{\rmap_j}$.
  \end{enumerate}

\noindent\textbf{Claim 1.} 
Thanks to the first-order rewritability of \dllitea, $\K^p$ and $\K^+$ can be
correspondingly represented as a Data-Centric Dynamic System in the
sense of \cite{BCDDM13}. The proof is then directly obtained from the
proof of Lemma 4.1 in \cite{BCDDM13}.

\smallskip
\noindent
\textbf{Claim 2-5.} The claims can be easily shown by observing that
$\K^p$ and $\K^+$ never produce an ABox that is $\T_p$-inconsistent,
since they only consider positive inclusion assertions. Consequently,
under each of the repair semantics, the repair service does not affect
the current ABox: it simply generates a unique successor that contains
the same ABox and service call map produced by the previous action
step. This shows that $\mf{\K^p}^b = \mf{\K^p}^c = \mf{\K^p}^{eb} =
\mf{\K^p}^{ec}$ and that $\mf{\K^+}^b = \mf{\K^+}^c = \mf{\K^+}^{eb} =
\mf{\K^+}^{ec}$. To get the claims, given the current state
$\tup{A,\rmap}$ in $\mf{\K^p}^b$, we specifically discuss the case in which
$\temp \not\in A$, and the case in which
$\temp \in A$:
\begin{compactitem}
\item[\em (base case)] Trivial, because the initial states of
  $\mf{\K^p}^b$ and $\mf{\K^+}^b$ coincide (they are both equal to
  $\tup{A^p_0,\emptyset}$).
\item[\em (case 1 - action step)] Since it cannot be the case that the
  state produced after an action step is $\T_p$-inconsistent, then the
  proof exactly follows the one for Claim 1.
\item[\em (case 2 - repair step)] Consider $\tau(i) = \tup{A,\rmap}$
  and $\tau^+(i) = \tup{A^+,\rmap^+}$ s.t.:
\begin{inparaenum}
\item $\temp \in A$ and $\temp \in A^+$;
\item $A$ and $A^+$ satisfy
  condition 1;
\item $\rmap$ and $\rmap^+$ satisfy conditions 2 and 3.
\end{inparaenum}
Since $A$ and $A^+$ are $T_p$-consistent, then there is a unique
successor $\tup{A - \{\temp\},\rmap}$ of $\tau(i)$ in
$\mf{\K^p}^b$, and a unique successor $\tup{A^+ - \{\temp\},\rmap^+}$ of $\tau^+(i)$ in
$\mf{\K^+}^b$. It is trivial to see that these successors satisfy the
three conditions of the claim above.
\end{compactitem}

\end{proof}

\begin{lemma}\label{PathDomination2}
For any KAB $\K$, we have that:
\begin{compactenum}
\item $\mf{\K}^s$ is dominated by $\mf{\K^p}^s$;
\item $\mf{\K}^b$ is dominated by $\mf{\K^p}^b$;
\item $\mf{\K}^c$ is dominated by $\mf{\K^p}^c$;
\item $\mf{\K}^{eb}$ is dominated by $\mf{\K^p}^{eb}$;
\item $\mf{\K}^{ec}$ is dominated by $\mf{\K^p}^{ec}$.
\end{compactenum}
\end{lemma}

\begin{proof}
We discuss each claim separately, by referring to the three inductive
conditions defined in the stronger claim of the proof of
Lemma~\ref{PathDomination}.

\smallskip
\noindent
\textbf{Case 1.} Trivial, because $\mf{\K}^s$ is a fragment of
$\mf{\K^p}^s$: it does not contain the portions of $\mf{\K^p}^s$ that
are generated starting from a $\T$-inconsistent (but always
$T_p$-consistent) ABox.

\smallskip
\noindent
\textbf{Case 2.} The base case is trivial, because the initial state
of $\mf{\K}^b$ is $\tup{A_0,\emptyset}$, the initial state of
$\mf{\K^p}^b$ is $\tup{A_0^p,\emptyset}$, and by construction $A_0
\subseteq A_0^p$.

The inductive case for an action step can be proven exactly in the
same way discussed in the proof of Lemma~\ref{PathDomination} - Claim
1.

We then focus on the inductive case for a repair step. Consider
$\tau(i) = \tup{A,\rmap}$ in $\mf{\K}^b$ and $\tau^p(i) =
\tup{A^p,\rmap^p}$ in $\mf{\K^p}^b$, s.t.~conditions 1, 2 and 3
hold. By construction, we know that:
\begin{compactitem}
\item every successor of $\tup{A,\rmap}$ in $\mf{\K}^b$ has the form
  $\tup{A',\rmap}$, where $A' \in \Rep(\A - \temp, \T)$;
\item $\tup{A^p,\rmap^p}$ has a unique successor $\tup{A^p -
    \{\temp\},\rmap^p}$ in $\mf{\K^p}^b$.
\end{compactitem}
Since the service call maps do not change, the successors continue to
obey to conditions 2 and 3. Furthermore, by definition of $\Rep()$, we
know that $A' \subseteq A$ and, by hypothesis, that $A \subseteq
A^p$. Consequently, $A' \subseteq A_p$, and therefore also condition 1
is satisfied.

\smallskip
\noindent
\textbf{Case 3.}  The base case and the inductive case for an action
step are as in Case 2. We then focus on the inductive case for a repair step. Consider
$\tau(i) = \tup{A,\rmap}$ in $\mf{\K}^b$ and $\tau^p(i) =
\tup{A^p,\rmap^p}$ in $\mf{\K^p}^b$, s.t.~conditions 1, 2 and 3
hold. By construction, we know that:
\begin{compactitem}
\item $\tup{A,\rmap}$ has a unique successor
  $\tup{A',\rmap}$  in $\mf{\K}^b$, where $A' = \bigcap_{A^r \in \Rep(A-\{\temp\},T)} A^r$;
\item $\tup{A^p,\rmap^p}$ has a unique successor $\tup{A^p -
    \{\temp\},\rmap^p}$ in $\mf{\K^p}^b$.
\end{compactitem}
Since the service call maps do not change, the successors continue to
obey to conditions 2 and 3. Furthermore, by definition we have $A'
\subseteq A$ and, by hypothesis, we know that $A \subseteq
A^p$. Consequently, $A' \subseteq A_p$, and therefore also condition 1
is satisfied.

\smallskip
\noindent
\textbf{Case 4.} This case is directly obtained from Case 2, and from the
observation that, by construction, each ABox of the consistent
approximant contains all the possible $\textsf{Viol}$ assertions,
since they are asserted in the initial state, and copied by means of a
specific effect contained in each of its actions. Therefore, after a
repair step, it is guaranteed that the ABox obtained in $\mf{\K}^{eb}$
is a subset of the corresponding ABox in $\mf{\K^p}^{eb}$.

\smallskip
\noindent
\textbf{Case 5.} This case is directly obtained from Case 3 and the
observation done for Case 4.
\end{proof}

The proof of Theorem~\ref{thm:itkabs-dec} is finally obtained by
combining Theorem~\ref{thm:wa-dominant} and the composition of
Lemma~\ref{PathDomination2} with Lemma~\ref{PathDomination}, thanks to
transitivity of domination.

\section{KABs with Consistent Query Answering}
As mentioned in the conclusion of the paper, an orthogonal approach to
manage inconsistency would be to make the KAB itself
inconsistency-tolerant. More specifically, we can conceive a KAB that
admits inconsistent ABoxes, and that replaces the standard
query answering service with an inconsistency-tolerant querying
service, able to extract meaningful answers even in presence of
inconsistent information.

In the following, we rely for this purpose on the standard notion of
\emph{consistent query answering} in databases \cite{Bert06}, which
has been extended to the knowledge base setting in
\cite{LLRRS10}. More specifically, we introduce the following query
answering service, which corresponds to the notion of AR-consistent
entailment in \cite{LLRRS10} (Definition 3). 

Given an UCQ $q$, the \emph{consistent-query answer} to $q$ over $(T,A)$ is the set $\cqa{q,T,A}$ of
substitutions $\sigma$ of the free variables of $q$ with constants in
$\adom{A}$ s.t., for every repair $A_r \in \Rep(A,T)$, 
$q\sigma$ evaluates to true in every model of $(T,A_r)$. Observe that,
when $A$ is $T$-consistent, the consistent-query answers coincide with
the certain answers.

Like for certain answers, we extend the notion of consistent-query
answer to ECQ as follows:  given an ECQ $Q$, the
\emph{consistent-query answer to $Q$ over $(\T,\A)$}, is the set $\Cqa{Q,\T,\A}$ of tuples
of constants in $\adom{A}$ defined by composing  the consistent-query answers
$\cqa{q,T,A}$ of UCQs $q$ through first-order
constructs, and interpreting existential variables as ranging over $\adom{A}$.

\subsection{Inconsistency-tolerant KABs} We introduce the
\emph{inconsistency-tolerant} semantics for KABs as the variation of
the standard semantics where:
\begin{compactitem}
\item all queries are answered using
consistent-query answering instead of certain answers (i.e., by
replacing every
$\Ans{Q,T,A}$ with $\Cqa{Q,T,A}$);
\item an action with parameters is applied even if the resulting ABox
  is $T$-inconsistent (in fact, consistent-query answering makes it
  possible to query such an inconsistent ABox in a meaningful way).
\end{compactitem}
We call \emph{it-KAB} a KAB interpreted under the
inconsistency-tolerant semantics. Given an it-KAB $\K$, we denote with
$\mf{\K}^{it}$ the transition system describing its execution semantics.

In order to specify temporal/dynamic properties over it-KABs, also the
$\muladom$ logic must be adapted, making it able to query even
$T$-inconsistent ABoxes in a meaningful way. In particular, we
introduce the logic $\muladomcqa$ that is syntactically equivalent to
$\muladom$, but redefines the semantics of local EQL queries $Q$ as follows:
\[\MODA{Q} = \{s \in \Sigma\mid \Cqa{Q\vfo, T, \abox(s)} = \true\}\]

\subsection{Verification of Inconsistency-Tolerant KABs}
In this Section, we show that the decidability results presented for
the repair semantics seamlessly apply to it-KABs as well.

\begin{lemma}
\label{lemma:iso-query-cqa}
Consider two knowledge bases $(T,A_1)$ and $(T,A_2)$, s.t.~there
exists a bijection $h$ with $A_2 = h(A_1)$. For every EQL query $q$,
we have $\tup{d_1,\ldots,d_n} \in \Cqa{q,T,A_1}$
  iff $\tup{h(d_1),\ldots,h(d_n)} \in \Cqa{h(q),T,h(A_1)}$.
\end{lemma}
\begin{proof}
This result is a direct consequence of the combination of
Lemmas~\ref{lemma:iso-query} and \ref{lemma:iso-rep}.
\end{proof}

\begin{theorem}
\label{thm:dec-it}
Verification of $\muladomcqa$ properties over run-bounded it-KABs is decidable.
\end{theorem}

\begin{proof}
By inspecting the proofs of Theorem~\ref{thm:dec} (given in
Appendix~\ref{sec:dec}), we observe that the possibility of
constructing a faithful finite-state abstraction for a run-bounded
KAB depends on the fact that its execution semantics produce bisimilar
runs starting from isomorphic states. This key property, in turn,
relies on the fact that the query answering service does not
distinguish between isomorphic states. Since this holds for
consistent-query answers as well (see
Lemma~\ref{lemma:iso-query-cqa}), we can follow, step-by-step, the same
proof given in Appendix~\ref{sec:dec}.
\end{proof}

\begin{theorem}
Given a weakly acyclic KAB $\K$, we have that $\mf{\K}^{it}$ is run-bounded.
\end{theorem}

\begin{proof}
Consider the consistent approximant $\K^p$ of $\K$. From
Lemma~\ref{PathDomination}, we know that $\mf{\K^p}^s$ is dominated by
$\mf{\K^+}^s$. By inspecting the proof of this claim, which in turn
refers to the proof of Lemma 4.1 in \cite{BCDDM13}, we know that
this is the case because, state by state, the answers extracted by $\K^p$
are contained in the ones extracted by
$\K^+$. 

We now observe that, by definition, given a TBox $T$, an ABox
$A$ and an EQL query $Q$, $\Cqa{Q,T,A} \subseteq \Cqa{Q,T_p,A} =
\Ans{Q,T_p,A}$. The equality $\Cqa{Q,T_p,A} =
\Ans{Q,T_p,A}$ holds because every ABox is consistent with $T_p$, and
the only repair of a consistent ABox is the ABox itself.

Consequently, we can apply the same line of reasoning used in the
proof of Lemma 4.1 in \cite{BCDDM13}, showing that $\mf{\K}^{it}$ is
dominated by $\mf{\K^p}^{s}$. By applying Lemma~\ref{PathDomination} and
transitivity of domination, this in turn implies that $\mf{\K}^{it}$ is
dominated by $\mf{\K^+}^{s}$. By recalling
Theorem~\ref{thm:wa-dominant} we finally get the result.
\end{proof}

\end{document}